\documentclass[11pt, a4paper]{article}
\usepackage{lineno,hyperref}

\usepackage[hmarginratio=1:1,top=25mm,right=18mm,left=18mm,bottom=25mm,columnsep=20pt, headsep=7mm]{geometry}

\usepackage[colorinlistoftodos]{todonotes}

\providecommand*{\Dom}[1]{{\rm dom}}							

\newcommand{\bracket}[1]{\left({#1}\right)}                    
\newcommand{\curlyb}[1]{\left\{{#1}\right\}}                    
\providecommand*{\N}[1]{\left\|{#1}\right\|} 					

\newcommand{\xbs}{{\boldsymbol{\mathsf{x}}}}

\providecommand{\bbR}{\mathbb{R}}

\usepackage{subfig}
\usepackage{xcolor}
\usepackage{bbm}
\definecolor{ForestGreen}{RGB}{34,139,34}

\newcommand{\newmcommand}[2]{\newcommand{#1}{{\ifmmode {#2}\else\mbox{${#2}$}\fi}}}
\newcommand{\newmcommandi}[2]{\newcommand{#1}[1]{{\ifmmode {#2}\else\mbox{${#2}$}\fi}}}
\newcommand{\newmcommandii}[2]{\newcommand{#1}[2]{{\ifmmode {#2}\else\mbox{${#2}$}\fi}}}
\newcommand{\newmcommandiii}[2]{\newcommand{#1}[3]{{\ifmmode {#2}\else\mbox{${#2}$}\fi}}}

\newmcommandi{\paren}{\left({#1}\right)}

\newcommand{\x}{x}
\renewcommand{\v}{{v}}

\newcommand{\reals}{\bbR}

\newcommand{\W}{{W}}

\newcommand{\Laplacian}{{L}}

\newcommand{\M}{M}
\newcommand{\X}{{X}}

\newcommand{\ind}{\mathbbm{1}}
\newcommand{\R}{\mathbb{R}}
\newcommand{\Ev}{\mathbb{E}}

\usepackage{amsfonts, amsmath, amssymb, amsthm}
\usepackage{mathtools}
\usepackage{enumerate,enumitem}
\usepackage{upgreek}
\usepackage{bbm}

\usepackage{xcolor}
\usepackage{graphics}

\usepackage{array}
\usepackage{etoolbox}

\usepackage{caption,float}
\usepackage{bm}
\usepackage[section]{placeins} 

\usepackage{multirow}

\usepackage{siunitx}
\usepackage{textcomp} 
\usepackage{gensymb}

\usepackage{algorithm}
\usepackage{algorithmic}

\theoremstyle{plain}
	\newtheorem{theorem}{Theorem}[section]
	\newtheorem{proposition}[theorem]{Proposition}
	\newtheorem{lemma}[theorem]{Lemma}
	\newtheorem{corollary}[theorem]{Corollary}
	
	\newtheorem{remark}[theorem]{Remark}
	\newtheorem{definition}[theorem]{Definition}

\allowdisplaybreaks

\usepackage{layout}
\numberwithin{equation}{section}

\begin{document}

\title{\vspace{0.8cm}Structure from Voltage\vspace{0.5cm}}

\newcommand{\footremember}[2]{%
    \footnote{#2}
    \newcounter{#1}
    \setcounter{#1}{\value{footnote}}%
}
\newcommand{\footrecall}[1]{%
    \footnotemark[\value{#1}]%
} 

\author{%
    Robi Bhattacharjee\footremember{ucsdRobi}{Department of Informatics, University of California San Diego, 9500 Gilman Dr, La Jolla, CA 92093, United States}%
    \and
  Alexander Cloninger\footremember{ucsdAlex}{Department of Mathematics, University of California San Diego, 9500 Gilman Dr, La Jolla, CA 92093, United States}%
  \and 
  Yoav Freund\footremember{ucsdYoav}{Department of Computer Science and Engineering, University of California San Diego, 9500 Gilman Dr, La Jolla, CA 92093, United States}%
  \and
 Andreas Oslandsbotn\footremember{uio}{Department of Informatics, University of Oslo, Problemveien 7, 0315 Oslo, Norway}\footremember{simula}{Simula Research Laboratory, Kristian Augusts gate 23, 0164 Oslo, Norway}%
  }
\date{}

\maketitle
\begin{abstract}
Data is often represented as point clouds embedded in high-dimensional space, with an intrinsic structure that concentrates on or close to lower-dimensional sets. Dimensionality reduction is a field dedicated to uncovering these lower-dimensional structures to mitigate the curse of dimensionality. Since the intrinsic structure can be highly non-linear. Non-linear dimensionality reduction techniques are necessary. 

A prominent technique for non-linear dimensionality reduction is spectral graph embeddings such as Laplacian eigenmaps (LP). The challenge with these techniques is that they rely on calculating the eigenfunctions of a large matrix that scales with the number of data points. These operations are expensive and hard to parallelize and typically require the data to be stored in memory. 

In this paper, we propose a scalable non-linear dimensionality strategy that is easy to parallelize. The approach we propose is based on the notion of a localized voltage function defined on a graph constructed on the point cloud. This is closely connected to the strategy used to define the effective resistance (ER) on graphs. Unfortunately, it has been shown that when vertices correspond to a sample from a distribution over a metric space, the limit of the ER between distant points converges to a trivial quantity that holds no information about the graph's structure. 

To circumvent this, we propose the notion of \textit{grounded resistor graphs}, in which the source vertex in ER is replaced with a source region, and the sink vertex is replaced with a universal ground vertex that is connected to all points by a fixed constant resistor. We then show that the energy-minimizing voltage over this new construction converges towards a non-trivial solution in the large sample limit. These voltage solutions are both localized near their source regions and can be solved independently. Finally, we present preliminary results, both theoretical and numerical, demonstrating how these solutions can be used to provide low-dimensional embeddings for the underlying space. 
\end{abstract}



\section{Introduction}
Dimensionality reduction is the problem of finding a low dimensional
representation for locations in a point cloud. The importance of
finding efficient algorithms to solve this problem has increased with
the availability of datasets with millions of data points and thousands
of dimensions.

Principal component analysis (PCA) is a very effective and efficient
algorithm for {\em linear} dimensionality reduction. In particular,
PCA depends on the shape of the point cloud only through the
covariance matrix, and the transformations it produces are linear. However, large and complex point clouds typically
exhibit non-linear structures, which makes traditional PCA
insufficient. Because of this, the field of {\em non-linear}
dimensionality reduction (NLDR) has emerged as a very active and
productive area of research (see Section~\ref{sec:otherwork}).

One approach to NLDR is to view the point cloud as a sample from an
underlying distribution and to characterize the Laplace-Beltrami
operator (LBO) over that distribution. Two prominent representatives
of this line of work are Eigen
maps~\cite{belkin2003laplacian} and Diffusion
maps~\cite{coifman2005geometric}, both of these algorithms are based
on finding the eigenfunctions of the LBO.

Computing these eigen-functions for a point cloud of size $n$ requires
computational time of $O(n^3)$ on one computer and is hard to
parallelize. As a result, finding these vectors for point clouds
larger than 100,000 is impractical with current
computers. Furthermore, typical implementations of the calculations of
these vectors require {\em all} of the data to be stored in a single
computer.  The underlying reason that all of the data has to be
available at the same time is that typical eigen-functions are
``global'', by which we mean that for most eigen-functions $f$ one can
find pairs of points $a,b$ such that the distance $d(a,b)>r$ for some
large value of $r$ and at the same time $|f(a)|>\gamma$ and
$|f(b)| > \gamma$ for some large $\gamma$.  In this work, we suggest
characterizing the LBO on the point cloud using {\em local} functions
instead. We say that a function $g$ is local if for any pair of points
$a,b$ such that $d(a,b)>r$ either $|g(a)|<\epsilon$ or
$|g(b)|<\epsilon$ for some small $\epsilon$.

The eigenvectors of the LBO are solutions of {\em homogeneous}
systems of equations. Meanwhile, to construct {\em local} functions, one needs
{\em non-homogeneous} solutions, i.e., pointwise
constraints. To help our intuition, we follow Doyle and
Snell~\cite{doyle1984random} in using resistor circuits to model the graph Laplacian. In this setting, one can define the non-homogeneous constraints by
choosing two vertices: a ``source vertex'' whose voltage is fixed to one, and a
``sink vertex'' whose voltage is fixed to zero. The remaining vertices are
``floating'' or ``unconstrained'' and the voltages on those nodes are
set by minimizing the energy dissipated by the circuit.

The minimal energy solution determines the current flowing from the
source to the sink. In turn, the reciprocal of the current determines
the {\em effective resistance} (ER). The ER is well-defined for discrete graphs. However, it becomes trivial and uninformative when considering point clouds in a metric space and letting the number of points increase to infinity, as shown by Von-Luxburg et al.~\cite{von2010getting}. The first contribution of this paper is to show that this problem can be alleviated by considering the
effective resistance between pairs of small {\em regions} rather than
pairs of {\em points}. Keeping the regions fixed as the number of
points increases to infinity produces non-trivial limits that can be
used for NLDR.

A natural approach at this point is to use the voltage functions for
different source-sink location pairs as an alternative to the
eigenvector representation. Unfortunately, the resulting voltage
functions are not local if the sink and the source are far from each
other. To ensure that the functions are local, we introduce a particular
form of the resistor graph, namely the ``grounded resistor
graph''. A resistor graph is transformed into a grounded resistor
graph by adding a single node, the ``ground'', connected
to each node in the original graph. The voltage functions are
generated by selecting a source and using the ground as the sink. The
result is voltage functions that are one at their respective source and strictly decreasing away from it.

We consider sources to be {\em landmarks} and use the associated
voltage function to measure the divergence from the
landmark~\footnote{This divergence is not a metric because it is not
  symmetric. On the sum of the divergence from $A$ to $B$ with the
  divergence from $B$ to $A$ {\em is} a metric, and this metric is the
  effective resistance in the grounded graph.}. We propose using the
divergences from a small fixed set of landmarks as a
dimensionality-reducing mapping.  To represent a non-landmark
location, we measure its divergence from each landmark, and that set
of distances uniquely identifies the location of the point. Moreover,
If the point cloud lies on a
$k$-dimensional manifold, then the divergence to the nearest $k+1$
locations uniquely identify the location. Thus the identity of the
close landmarks, together with the distances from those landmarks, can
be used as a low-dimensional representation of the location.

Our contributions can be summarized as follows:
\begin{itemize}
\item We alleviate the problem described by Von-Luxburg et al.~\cite{von2010getting} by appropriately scaling the resistances as the sample size grows.  
\item We prove the existence, convergence to a non-trivial limit, and shape properties of the
grounded metric voltage function.
\item We derive an analytical solution for the grounded metric voltage
on the sphere and support these solutions by numerical experiments.
  \item We show the results of a few experiments on real-world data sets, including MNIST and
  Frey-face data sets.
\end{itemize}

The rest of the paper is organized as follows.
In section 2, we introduce the notion of a grounded resistor graph.
While section 3 extends this to the metric graph and the grounded resistor graph on metric spaces.
In section 4, we discuss the limitations of LE and ER in the large-data metric-graph limit and compare them to our method, which overcomes these limitations.
In section 5, we theoretically analyze the grounded metric graph and show existence, convergence, and shape properties of the voltage solution. In section \ref{section:Embedding using localized voltage functions}
we show theoretically how the voltage solution can be used to embed
the sphere and support our result with numerical experiments. In
Section \ref{section:A note on computational efficiency} we discuss
computational advantages with the metric voltage function while in
Section \ref{section:Future work} we conclude and discuss future work.

\subsection{Relations to other work}
\label{sec:otherwork}
In this section, we give a brief overview of related work.
\subsubsection*{Related work on non-linear dimensionality reduction}
The published literature on NLDR is vast; see the survey~\cite{van2009dimensionality}. 
Below is a non-exhaustive list of some of the significant approaches.

\begin{itemize}
\item {\bf Kernel based} The kernel-PCA~\cite{scholpf1998kernelpca} method generalizes
  the classical linear PCA to non-linear problems.
\item {\bf Manifold based} methods rely on the assumption that the
  point cloud lies on a low dimensional manifold and use ideas from
  differential geometry. These include
  ISOMAP~\cite{tenembaum2000aglobal} that uses shortest paths as an
  approximation to geodesics, and Locally linear
  embeddings~\cite{roweis2000lle} uses local linear approximations of
  manifold.
\item {\bf Optimization based} Direct optimization method use gradient
  descent algorithms to improve an initial embedding.  These include
  t-SNE~\cite{van2008visualizing}, UMAP~\cite{McInnes2018UMAP},
  force-based algorithms \cite{Steinerberger2022}, and LDLE
  \cite{kohli2021ldle}. 
\item {\bf Laplacian based} are based on the Laplace-Beltrami operator. These include Laplacian
eigen-maps~\cite{belkin2003laplacian,belkin2008towards} and Diffusion
maps~\cite{coifman2005geometric} 
\end{itemize}
There are various overlaps and combinations of these categories. For
example, it has been shown that Laplacian based methods correspond to
kPCA with particular choices of the kernel
\cite{ham2004kernel}. Similarly, Laplacian based methods have been used to
initialize optimization based methods~\cite{kobak2021initialization}.

\subsubsection*{Related work on effective resistance}
The ER was introduced as a distance function on graphs by Klein et al.~\cite{klein1993resistance}. Since then, it has proven a useful tool for capturing structural characteristics of graphs, with numerous applications, such as phylogenetic networks
\cite{forcey2020phylogenetic}, detecting community structure
\cite{zhang2019detecting}, distributed control
\cite{barooah2006graph}, graph edge sparsification
\cite{spielman2011graph}, and measuring cascade effects
\cite{tauch2015measuring} e.g. in power grids
\cite{kocc2014impact, wang2015network, cavraro2018graph}.


\section{Resistor Graphs}
\label{section:Grounded Resistor Graphs}

In this section, we introduce the notion of a \textit{grounded resistor graph}, which will serve as our fundamental tool for analyzing graphs along with the spaces from which they are sampled. We first review several key ideas and concepts about \textit{resistor graphs}.

\subsection{General Resistor Graphs}
\label{section:resistor_graphs}
Let $(\X,\W)$ be an undirected, weighted graph with nodes
$\X = \{x_1,\ldots, x_n\}$, and edge weights $\W_{i,j}$. Let the {\em
  voltage} $\v$ be a function $\v:\X\rightarrow\reals$. The {\em
  energy} of the voltage $\v$ is defined as
 \begin{equation} \label{eqn:energy}
   E(\v) \doteq  \sum_{x_i, x_j \in X} \W_{i,j} (\v(x_i)-\v(x_j))^2 =
   \v^T \Laplacian \v
 \end{equation}
where $\Laplacian = D - \W$ is the Laplacian matrix, and $D$ is a diagonal matrix such that $D_{ii} = \sum_j \W_{ij}$.
 
It is easy to see that, with no constraints, the minimal energy is
zero, and that zero energy is attained for any $\v$ where all of the
entries are equal; $v(x_i) = v$ for all vertices in the graph.
Therefore, to obtain a non-trivial voltage distribution on the graph, it is necessary to constrain the system. This motivates 
the Energy Minimizing Voltage (EMV), which is the main object we study in
this paper. 

\begin{definition}[EMV]
The energy-minimizing voltage of a weighted graph $(X, W)$ with respect to source and sink nodes $X^s$ and $X^g$, is the solution to the following optimization problem:
\begin{align*}
    \begin{split}
        \min_{v} & \quad \sum_{x_i, x_j \in X} \W_{i,j} (\v(x_i)-\v(x_j))^2 \\
        \text{Subject to} & \quad  v(x_i) = 1 \quad \forall x_i \in X^s, \quad  v(x_i) = 0 \quad \forall x_i \in X^g,
    \end{split}
\end{align*}
where $\X^s$ are the source nodes and $\X^g$ the sink nodes.
\label{def:EMV_for_LVE}
\end{definition}

\noindent The EMV owes its name to a helpful interpretation of the graph as an electrical network. In the following, we give a brief description of this view, while more details can be found in Doyle and Snell \cite{doyle1984random}. 

\paragraph*{Electric networks:} The EMV can be thought of as the voltage in an electrical network, where the graph corresponds to the underlying electric circuit. In this view, each vertex
has an associated {\em voltage} $\v(x_i)$ and each edge $(x_i,x_j)$ has
associated a non-negative {\em resistance} $R_{i,j}=\frac{1}{W_{i,j}}$ 
and a signed {\em current} $J_{i,j}=-J_{j,i}$. Furthermore, Ohm's law relates the current and resistance at an edge with the voltages at the vertices connected by the edge. The relation can be written as 
\begin{equation}
      \v(x_i)-\v(x_j) = R_{i,j} J_{i,j} \quad \text{or alternatively} \quad  J_{i,j} = W_{ij}(\v(x_i)-\v(x_j)).
\end{equation}
Meanwhile, from Kirchhoff's law, the sum of currents entering a node $i$ must be zero, namely
\begin{equation}
\sum_{j\sim i} J_{i,j} = J_{ext,i},
\end{equation}
where $I_{ext, i}$ is an external current that can be either a source, a sink or zero if the node is un-constrained (no external source applied). Combining these laws we have that
\begin{equation}
    (Lv)_i =\sum_{j\sim i} W_{ij}(v(x_i) - v(x_j)) = J_{ext,i}
    \label{eq:combined_kirchhoff_and_ohm}
\end{equation}
from which it is easy to see that the EMV in Definition \ref{def:EMV_for_LVE} can be attained. In particular, the constraints on the source and ground voltages can be enforced by the external current.

Because of this, we can interpret Equation~(\ref{eqn:energy}) as the energy dissipated in a circuit in the form of heat when no power is injected. With no external source, it is clear that this energy is zero. Meanwhile, the constrained system in Definition~(\ref{def:EMV_for_LVE}) corresponds to a circuit connected to an external source for which a non-trivial minimal energy voltage vector exists.


\subsection{Grounded Resistor Graphs}
Computing the EMV over arbitrary choices of sources $\X^s$ and sinks $\X^g$ can reveal aspects of the global structure of a graph -- for example, measuring the total current that flows from $\X^s$ to $\X^g$ can give a measure of the connectivity between these two sets. In this work, we are particularly interested in the case where $\X^s$ is a small set of closely connected vertices, and $\X^g$ is selected to reveal the \textit{local} structure of the graph around $\X^s$. Motivated by this, we introduce the idea of a \textit{grounded resistor graph}, which replaces the set of sink nodes $\X^g$ with a dummy node $g$ that represents a universal sink. We refer to this sink as the grounding node.

\begin{definition}[Grounded resistor graph] \label{def:grounded_resistor_graph}
Let $(X, W)$ be a weighted graph. Then its grounded graph with grounded weight $\rho$, $(X, W, \rho)$, is the graph in which the extra node, $g$, is connected to all vertices with an edge of weight $\rho$. 
\end{definition}

Here, we don't consider $g$ as a node of $X$ as its behavior is completely determined by the weight $\rho$. As we will see later, it will be convenient to keep $X, W$ unchanged when considering the grounded graph. 
Because everything is connected to the ground, the EMV over $(X, W, \rho)$ will naturally decay to $0$ on nodes far from the source. We define the grounded EMV as follows.

\begin{definition}[Grounded EMV]\label{def:grounded_EMV}
\begin{align*}
    \begin{split}
        \min_{v: X\rightarrow [0, 1]} & \sum_{i, j \in X} W_{ij}(v(x_i) - v(x_j))^2 + \sum_i \rho v^2(x_i) \\
        \text{Subject to} & \quad  v(x_i) = 1 \quad \text{for all} \quad x_i \in \X^s.
    \end{split}
\end{align*} 
\end{definition}

Here, the term $\sum_i \rho v(x_i)$ represents the amount of energy corresponding to the edges connecting each node to the ground vertex. Furthermore, we exclude the ground node from $v$ because it is always defined to have a voltage of $0$. For a given source $X^s$, we can describe the solution of the EMV using the \textit{grounded weight matrix} $\widetilde{D}^{-1}\widetilde{W}^{(s)}$. Here $\widetilde{D}\in \bbR^{n \times n}$ is a diagonal matrix with $D_{ii}=1$ for $x_i\in X^s$ and otherwise $\widetilde{D}_{ii} = \rho + \sum_{j = 1}^n W_{ij}$. Furthermore, $\widetilde{W}^{(s)}\in \bbR^{n \times n}$ is defined as
\begin{equation*}
    \widetilde{W}^{(s)}_{ij} = 
    \begin{cases}
     1, & \text{if} \quad i=j, x_i \in \X^s \\
     0, & \text{if} \quad i\neq j, x_i \in \X^s \\
     W_{ij}, & \text{otherwise}.
    \end{cases}
\end{equation*}
We note that the source is incorporated by $\widetilde{W}^{(s)}$ and $\widetilde{D}$ incorporates the effect of the grounding node through the additional $\rho$ in the sum. With the definition of the \textit{grounded weight matrix} we have from Lemma \ref{lemma:Solution_EMV_for_LVE} that the solution to the grounded EMV can be written as:

\begin{lemma}\label{lemma:powermethod} The solution to the EMV in Def. \ref{def:grounded_EMV} is the unique voltage function $v^*: X \to [0, 1]$, which satisfies $v^* = \widetilde{D}^{-1}\widetilde{W}^{(s)} v^*$ and $v^*(x_i) = 1$ for all $x_i \in \X^s$. 
\label{lemma:Solution_EMV_for_LVE}
\end{lemma}

The motivation for writing the EMV in this way will become clear later when we introduce the metric

\begin{remark}
We note that the sum of the $i$-th row of $\widetilde{W}^{(s)}$ is $1/(1 + \rho/\sum_{i} W_{ij})$ where $1/\rho$ is the resistance to ground. This relation highlights the importance of the ground node because $\rho/\sum_{i} W_{ij}\approx 0$ means a trivial solution since the rows sum to one. Therefore, as the large graph limit means increasing $\sum_{i}W_{ij}$, tuning of $\rho$ can be used as a counterweight. Furthermore, as we will show, the voltage decay away from the source is tightly linked to the magnitude of $\rho$.
\label{remark:rowsum_of_WHat}
\end{remark}



\section{Grounded Resistor Graphs over Metric Spaces}\label{section:Grounded Resistor Graphs over Metric Spaces} 

We focus on a special type of graphs, namely graphs constructed from samples drawn from a distribution over a metric space. Let $(\M, d)$ be a compact metric space and $\mu$ the probability measure over $\M$.
Let the kernel function $k: \M \times \M \rightarrow [0, 1]$ be a function that defines what it means for two points to be "near" each other. 
Two commonly used kernel functions are:
\begin{itemize}
\vspace{0.25cm}
    \item The radial kernel: $k_r(x, y) = \mathbf{1}(d(x, y) \leq r)$ where $r > 0$ is some fixed radius.
    \vspace{0.25cm}
    \item The Gaussian kernel: $k_\sigma(x,y) = \exp(\frac{-d(x,y)^2}{2\sigma^2}),$  where $\sigma > 0$ is the fixed temperature parameter.
\vspace{0.25cm}
\end{itemize}


Let $X_n = \{\x_1, \cdots, \x_n\}$ be a sample of points $\x_i \in \M$ drawn i.i.d. from $\mu$. Our main idea is to construct a weighted grounded graph, $(\X_n, W, \rho)$ by connecting points $\x_i, \x_j$ by weights $W_{ij} \propto k(\x_i, \x_j)$, and by utilizing a grounded weight proportional to $\rho$. Then, if $\M^s \subseteq M$ is a local region in $M$, we can leverage the EMV over the grounded graph to understand the structure of $M$ nearby $\M^s$. 

\begin{figure}[htb!]
\centering
\includegraphics[width=0.9\textwidth]{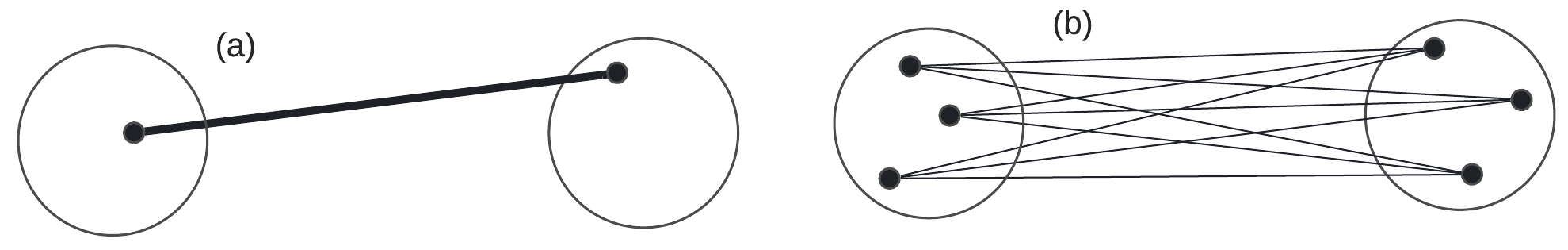}
\caption{Example of resistance scaling.  {\bf (a)} Number of edges connecting $T$ and $T'$ for $m=1$ point sampled from each region.  {\bf (b)} Number of edges connecting $T$ and $T'$ for $m=3$ points sampled from each region. Notice that the number of edges between these regions is $m^2$. } \label{fig:resistor_split}
\end{figure} 

To do so, it is essential that our computations converge towards a non-trivial solution as the number of sampled points, $n$, goes towards infinity. In particular, it is necessary for $W, \rho$ to appropriately scale as $n$ changes. It is therefore natural to demand that the physical properties of the graph, embodied by the resistance, current and voltage, should remain relatively stable as $n$ increases. Thus, it is crucial to understand how the edge resistances should scale with the number of points sampled.

\paragraph{Scaling based on regional density}
To this end, consider two small regions $T$ and $T'$ such that $k(x,x')>0$ for all $x\in T$ and $x'\in T'$.  Note, these are not necessarily sources and sinks, just two small regions.  For simplicity let $k(x,x')$ be constant, which means each edge has equal resistance $R = 1/k(x,x')$.  Our goal is to keep the resistances between these two regions constant as the number of points changes.  For a fixed $X_n$, on average there are $m$ points $x\in S_n = \curlyb{x\in X_n : x\in T}$ and $m$ points $x'\in S_n' = \curlyb{x\in X_n : x\in T'}$.  This results in $m^2$ edges between $T$ and $T'$.   This means the total resistance between these regions is $R/m^2$, given that these edges are connected in parallel.

The issue here is that, once we move to a denser sample $X_{2n}$ there will be, on average, $2m$ points in $S_{2n}$ and $S_{2n}'$ respectively.  This will create a net resistance $\frac{1}{4} R/m^2$, which means the resistance between these physical regions $T, T'$ is decreasing and will go to 0 as $n$ goes to infinity.
We illustrate this construction in Figure \ref{fig:resistor_split}.

Based on these considerations, we formally define the grounded metric resistor graph, illustrated in Fig. \ref{fig:groundedGraph}.

\begin{definition}[Grounded metric resistor graph]\label{definition:grounded_metric_resistor_graph}
Let $X_n = \{x_1, x_2, \dots, x_n\} \sim \mu$ be a set of points sampled from data distribution $\mu$ over $M$, $k: M \times M \to [0, 1]$ be a kernel similarity function, and $\rho_g$ be a fixed scaling constant for the grounded weight. Then the \textbf{grounded metric resistor graph}, $(X_{n}, W, \rho)$, is the weighted graph defined with grounded weight, $\rho = \frac{\rho_g}{n}$, and edge weights $(W)_{ij} = \frac{k(x_i, x_j)}{n^2}$. 
\end{definition}
Next, we can define the grounded EMV for a region $M^s$ by considering all points inside $M^s$ as sources. 

\begin{definition}[Grounded metric voltage function]
Let $M^s \subseteq M$ be any set, $\rho$ be a constant, and $(X_n, W, \rho)$ be the corresponding grounded resistor graph constructed from $X_n \sim \mu^n$. We define the \textbf{grounded metric voltage function} with respect to $M^s$, denoted $v_n^{*}:X_n\rightarrow  [0, 1]$, as the grounded EMV with respect to $X^s = M^s \cap X_n$. 
\label{def:grounded_metric_voltage_function}
\end{definition}

\begin{figure}[htb!]
\centering
\includegraphics[width=.6\textwidth]{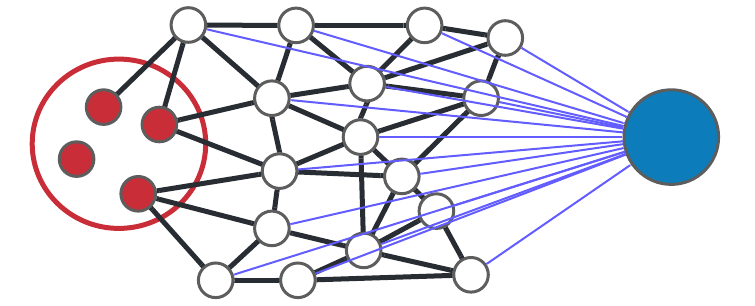}
\caption{An example of a grounded graph with red nodes denoting $\X^s$ and the blue node $g$ for the terminating ground node.  Note that all points are connected to the ground.}\label{fig:groundedGraph}
\end{figure}


\section{Outlining our method and comparing to existing work}
We propose to use the grounded voltage function from Definition \ref{def:grounded_metric_voltage_function} as an embedding tool. The idea is to distribute sources (landmarks) across the point cloud we want to embed and, for each source, solve for an independent grounded voltage function. With $d$-independent voltages, we can construct a $d$-dimensional embedding, a strategy we call Localized voltage eigenmaps (LVE).

In the remainder of this paper, we will establish the theoretical foundations for LVE. However, before we do this, we will outline two existing approaches, namely the effective resistance (ER) and the Laplacian eigenmaps (LE), and explain how our method improves on these schemes.

\subsection{Effective resistance}
The ER is a measure for calculating distances on graphs \cite{klein1993resistance}. It arises naturally from the electrical system interpretation by considering each node pair $(l, k)$ as a source-sink connection and enforcing a unit current between them. The standard formulation of ER follows by enforcing these constraints on Equation \eqref{eq:combined_kirchhoff_and_ohm}, which gives
\begin{equation}\label{eq:effective_resistance}
    R^{eff}_{lk} = v(x_l) - v(x_k) = (e_l -e_k)^\top L^{\dagger}(e_l -e_k) = \N{z_l - z_k}^2,
\end{equation}
where the effective resistance $R^{eff}_{lk}$ is an euclidean distance matrix \cite{ghosh2008minimizing}. The last term in Eq. \eqref{eq:effective_resistance} shows how the ER can be considered an n-dimensional embedding with $z_l = \Lambda^{-1/2}V^\top e_l$ for $L = V\Lambda V^\top$. The advantage of ER is that the distance between nodes can be computed without explicitly calculating the voltage functions $v$. Instead, it suffices to solve for $L^{\dagger}$, the pseudo-inverse of $L$. Methods have also been proposed for distributed computation \cite{gillani2021queueing}. 

\paragraph{Effective resistance as an EMV}
We can relate the ER to the grounded voltage function by re-formulating the ER as the solution to an EMV optimization problem. For a node pair $(l,k)$, the voltage function defining the effective resistance $R_{lk}^{eff} = v(x_l) - v(x_k)$ is the solution to the following minimization problem.
\begin{align}
    \begin{split}
     \min_{v: X \rightarrow \bbR} \quad & \sum_{i,j} W_{ij}(v(x_i) - v(x_j))^2 - \sum_{i\in (l,k)} J_i v(x_i)\\
     \text{Subject to} \quad & J_i = 1; \,\,  i = l, \quad J_i = -1; \,\,  i = k \quad \text{and} \quad J_i = 0 \quad \text{otherwise}.
    \label{eq:ER_optimization_problem}
    \end{split}
\end{align}
By considering Ohm's law $J_{ij} = (v_j - v_i)/R_{ij}$ with $v_j=0$, $R_{ij} = \rho$ and summing over all nodes instead of $(l,k)$, the relation to the grounded EMV becomes clear.

\paragraph*{Limitations with the effective resistance}
The problem with the ER, demonstrated by  \cite{von2010getting}, is that the distance between nodes connected by a path converges in the large graph limit to a trivial quantity which only depends on the degree of the end-nodes. 


\subsection{Laplacian eigenmaps}
Laplacian Eigenmaps is a widely used tool for finding a $d < n$ dimensional representation of the $n$ nodes on a graph. The goal is to assign coordinates $z_k \in \bbR^d$ to each node $k$, so nearby points on the graph remain
close. This is done by the following: for $l=1,\dots, d$ solve
\begin{align}
    \begin{split}
        \min_{v^{(l)}: X\rightarrow \bbR} \quad & \sum_{i,j} W_{ij}(v^{(l)}(x_i) - v^{(l)}(x_j))^2\\
        \text{Subject to} \quad & v^{(l)} \perp v^{(l^\prime)} \quad \text{and} \quad v^{(l)} \perp 1.
    \end{split}
    \label{eq:LE_optimization_problem}
\end{align}

To avoid the trivial solution $v^{(l)} = 1$, the constraint $v^{(l)} \perp 1$ is enforced. Meanwhile, the orthogonality constraint $v^{(l)} \perp v^{(l^\prime)}$ is introduced to avoid solving for the same function $d$ times. It turns out that the minimizing set, satisfying these constraints, are the first $d$ eigenfunctions of $L$, modulo the first, which turns the optimization into an eigenvalue problem.

\paragraph*{Laplacian eigenmaps as an electrical network} The eigenfunctions that solves the LE problem satisfies trivially $L^{(l)} = \lambda_{l} D v^{(l)}$. By thinking of the right-hand side as an external current $J_{ext, l} = \lambda_{l} D v^{(l)} = \lambda_l (d_1 v^{(l)}(x_1),\dots,d_n v^{(l)}(x_n))^\top$ we can think of each $v^{(l)}$ as a voltage over the nodes, that satisfies Kirchhoff's and Ohm's law through Equation \eqref{eq:combined_kirchhoff_and_ohm}. 

This means that the LE optimization problem in Equation \eqref{eq:LE_optimization_problem} can be considered an EMV defined over an electrical resistor graph, where constraints are: no current accumulation due to $v^{(l)} \perp 1$ and orthogonality between each voltage solution. Furthermore, the interpretation of the external current $J_{ext, i}$ is that for each voltage function $v^{(l)}$, the nodes in the graph can act as a source or a sink depending on the sign of $v^{(l)}(x_k)$ at node $k$. In particular, we note that there are no constraints on the locality of these sources and sink nodes. 

\paragraph*{Limitations with Laplacian eigenmaps}
An important limitation of LE is the orthogonality condition, which prevents the eigenvectors from being computed independently, preventing distributed computation. Furthermore, the Laplacian eigenvectors are typically global and, therefore, expensive to compute. By global, we mean in this context that the eigenfunctions have non-zero support almost everywhere on the graph, meaning we effectively need all nodes to compute the eigenfunctions. We can understand this global behavior from the electrical network interpretation because, as we have seen, the source and sink nodes are not confined to specific regions of the graph, and because of this, neither are the voltage solutions.

\subsection{Advantages of the grounded metric voltage function}
In this paper, we are motivated by making an embedding for metric spaces, which requires a computationally cheap and non-trivial solution in the large graph limit. As we have seen, both LE and ER face limitations in this limit; The ER suffers from a trivial solution; Whereas the LE is prevented from distributed computations and suffers from high computational complexity. Because of these limitations, both LE and ER are in their traditional form, prevented from being used in the metric graph setting. In this paper, we show that embedding using the grounded voltage function has the potential to overcome these limitations.

\paragraph{Overcoming the limitations of LE} As we show in Section \ref{section:Bounding the shape of the voltage function}, combining a localized source with a universal ground creates a localized voltage solution, reducing computational complexity. Furthermore, the grounded voltage functions are free of dependency conditions, such as the orthogonality condition enforced by LE, allowing distributed computations.

\paragraph{Overcoming the limitations of ER} As apparent from Eq. \eqref{eq:ER_optimization_problem} there is a close connection between the ER and the grounded EMV. Therefore, one could expect a trivial limit also for the grounded voltage function. However, as we show in Sections \ref{section:Convergence_of_the_grounded_metric_voltage_function} and \ref{sect:examples_grounded_resistor_graphs_over_metric_spaces}, the grounded metric voltage function converges to a non-trivial limit as the number of samples goes to infinity. This non-trivial limit is achieved because we think of each node in the graph as a region with a density instead of individual points, a view discussed in detail in Section \ref{section:Grounded Resistor Graphs over Metric Spaces}. In particular, this view introduces region-based scaling and sources that cover a finite region instead of the point sources used in the ER setting. 

Figure \ref{fig:er_compare} illustrates the difference between using a point source and a regional source whose strength increases proportionally with the number of samples.  We demonstrate this with four algorithms across different numbers of points sampled from the same metric space. All of these are done without the universal grounding node but with a source at $s=(0.1, 0.1)$ and a sink at $g=(0.7, 0.7)$.  Both source and sink regions are of radius $0.1$.  We compute voltage curves through the following four methods: (PM) the power method described in Lemma \ref{lemma:powermethod}; (RegionER) ER using a source vector $e_{X^s}$ where the voltage curve is 
\begin{align*}
v=L^\dagger \left(e_{X^s} - e_{X^g} - (p_s - p_g)\right), \quad e_{X^s}(x) = \begin{cases}1, & \textnormal{ if } x\in X^s\\0, & \textnormal{ else} \end{cases}, \quad e_{X^g}(x) = \begin{cases}1, & \textnormal{ if } x\in X^g\\0, & \textnormal{ else} \end{cases},
\end{align*}
where $p_s = \frac{1}{|X|}\sum_{x\in X} e_{X^s}(x)$ is the density of the source, same for $p_g$ and the sink. The mean is subtracted so that the total external current is mean 0; (DensityER) ER using an indicator source vector $e_s$ localized at the exact source node $s\in X^s$, an equivalent sink vector $e_g$, and the voltage curve $$v=L^\dagger \left(p_s\cdot e_{s} -p_g \cdot e_{g} - (p_s-p_g)\right),$$
where the mean is subtracted so that the total external current is mean 0; (ER) standard ER using the indicator source vector $e_s$, the equivalent sink vector $e_g$, and the voltage curve $v = L^\dagger (e_s - e_g)$.    It is clear from Figure \ref{fig:er_compare} that using a source region instead of a source point (even if that point is weighted by the local density) is critical to attaining a nontrivial limit as the number of points increases.  

\begin{figure}[htbp]
  \centering
  \footnotesize
  \begin{tabular}{cccc}
    \includegraphics[width=0.2\textwidth]{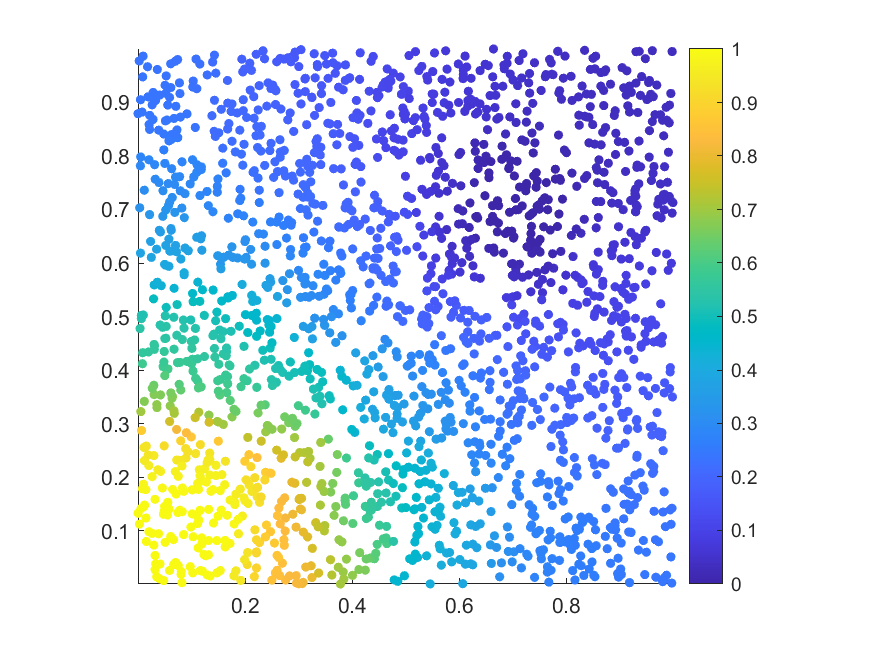} & 
    \includegraphics[width=0.2\textwidth]{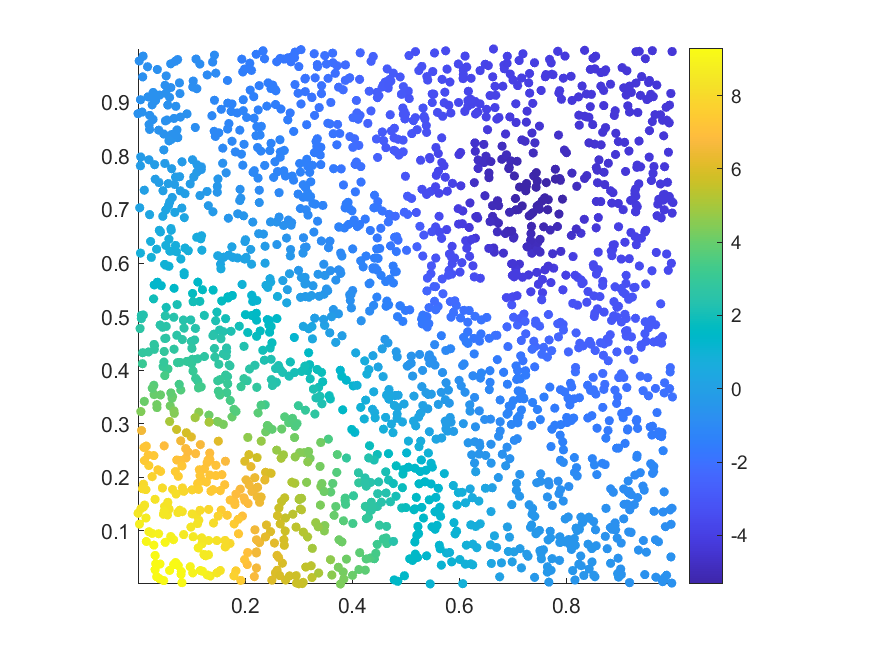} & 
    \includegraphics[width=0.2\textwidth]{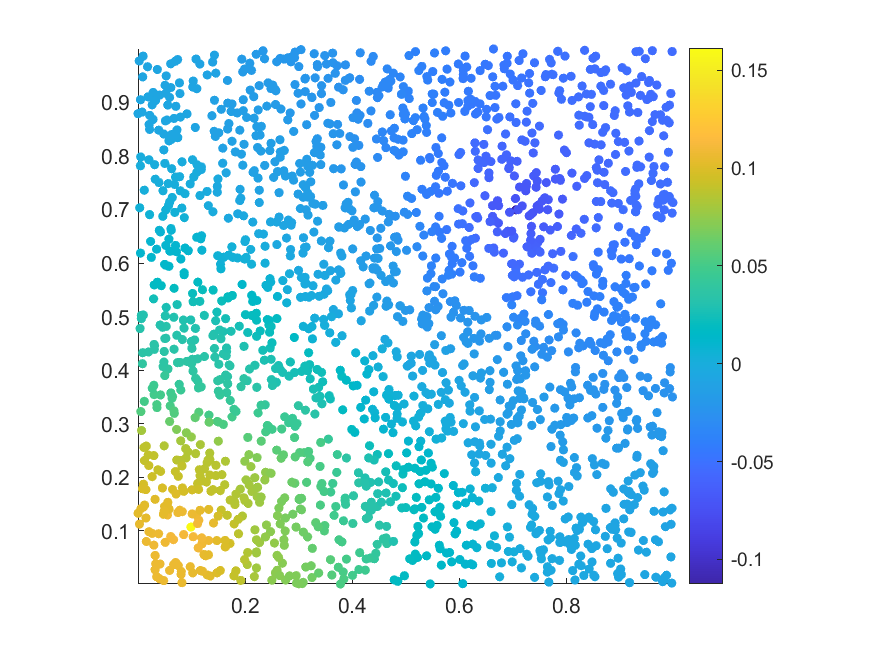} & 
    \includegraphics[width=0.2\textwidth]{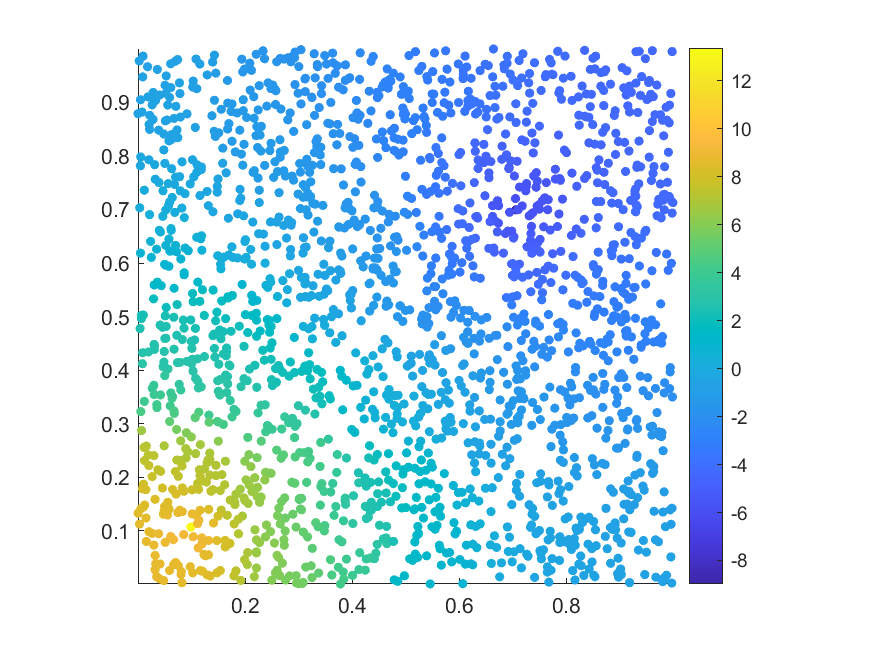} \\
    \includegraphics[width=0.2\textwidth]{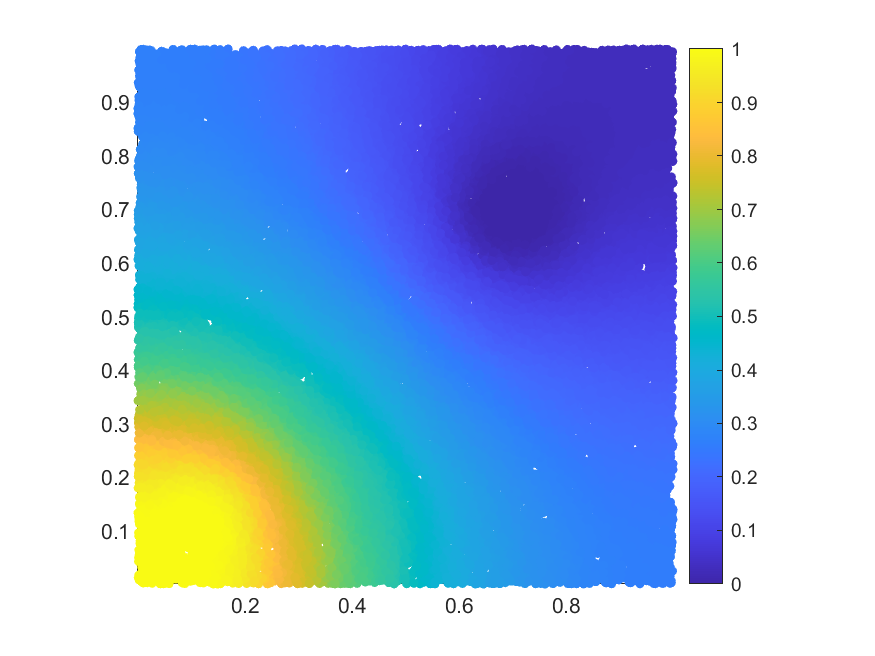} & 
    \includegraphics[width=0.2\textwidth]{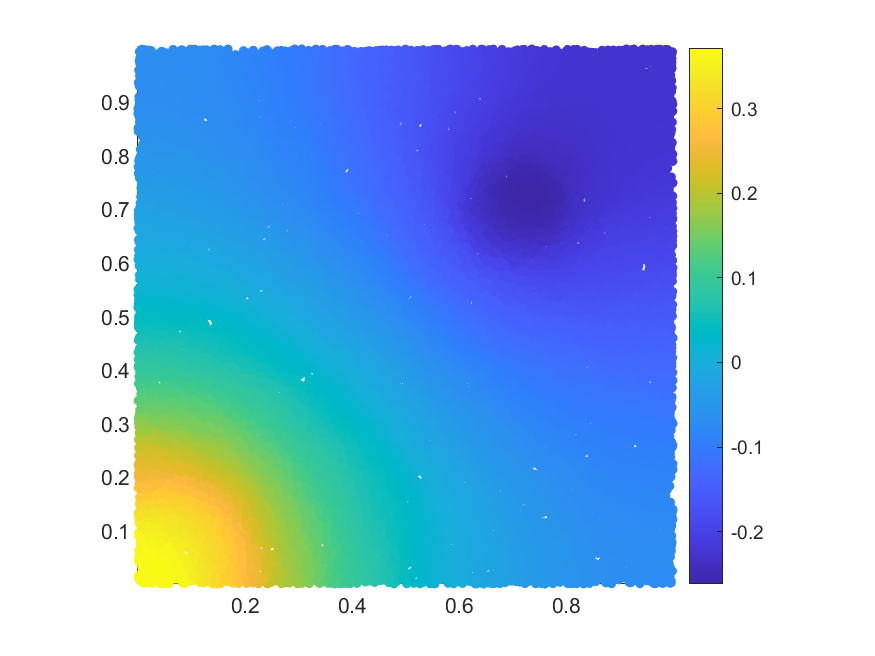} & 
    \includegraphics[width=0.2\textwidth]{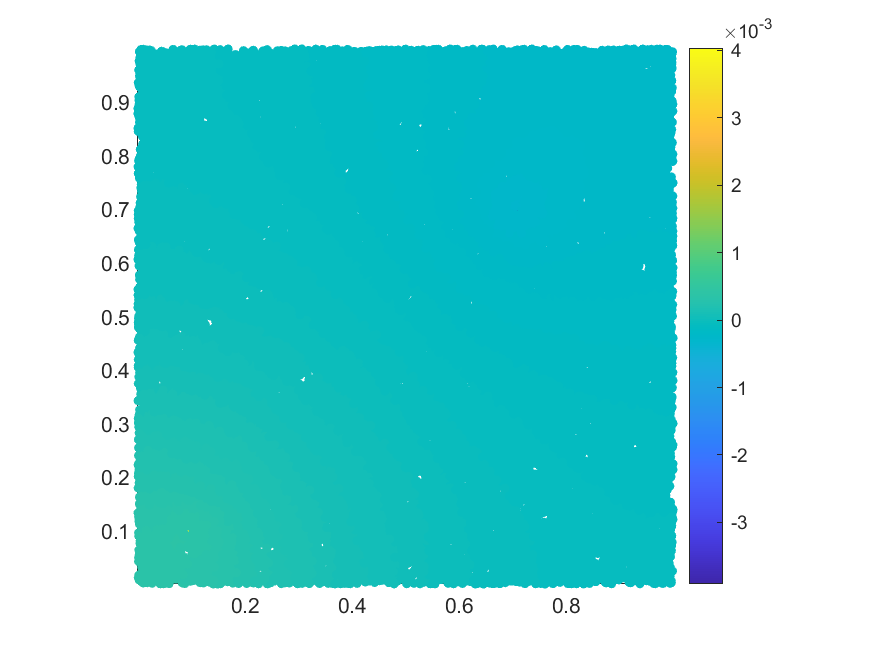} &
    \includegraphics[width=0.2\textwidth]{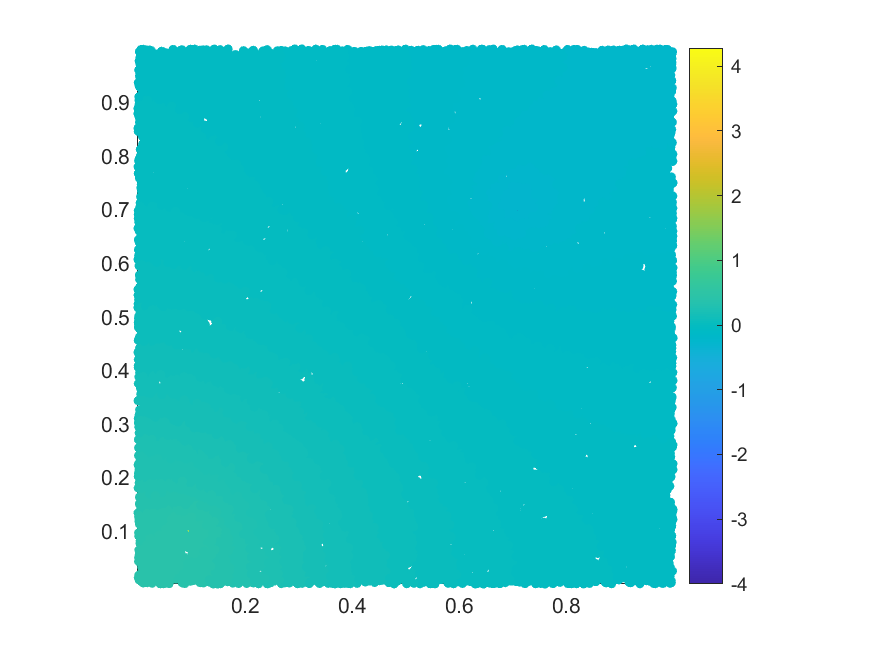} \\
    PM & RegionER & DensityER & ER
  \end{tabular}
  \caption{Trivial Limit of ER: (top) $2^{11}$ points, (bottom) $2^{15}$ points.}
  \label{fig:er_compare}
\end{figure}

Finally, Figure \ref{fig:Motivation} summarizes our contribution compared to the LE and ER schemes.

\begin{figure}[htbp]
  \centering
  \includegraphics[width=\textwidth]{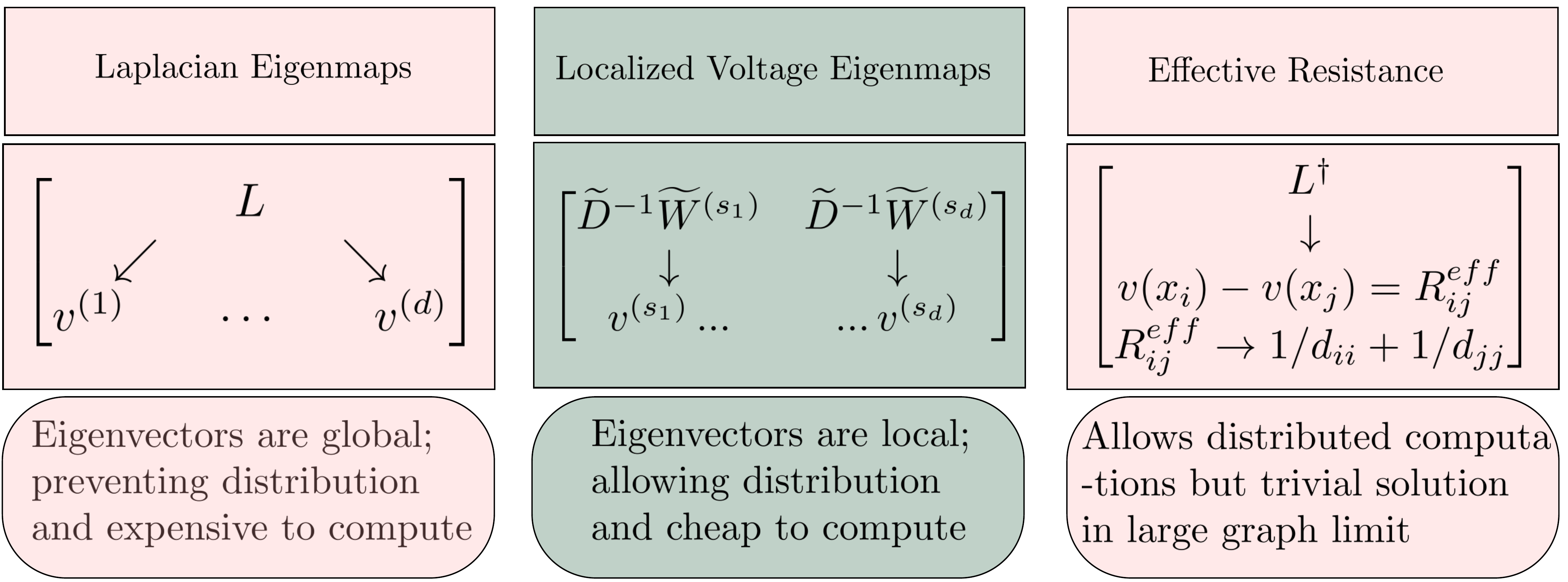}
  \caption{Illustration of our contribution}
  \label{fig:Motivation}
\end{figure}


\section{Analysis of the grounded metric voltage function}\label{sec:ground_analysis}
Let $v_n^*$ be the grounded metric voltage function from Def. \ref{def:grounded_metric_voltage_function}. In this section, we show that this voltage function converges to a non-trivial function as the sample size increase and provide shape bounds on the voltage decay away from the source.

\subsection{Convergence of the grounded metric voltage function}
\label{section:Convergence_of_the_grounded_metric_voltage_function}
We start by showing that $v_1^*, v_2^*, \dots $ converge as our sample size increases. Our first step is to define the voltage function formally, $\v^*: \M \to \reals$, that they converge to. To this end, we start with an explicit expression for $v_n^*$ which follows directly from Lemma \ref{lemma:Solution_EMV_for_LVE}.


\begin{proposition}\label{prop_ohms_law_result}
Let $M^s \subseteq M$ be a source subset, $X_n \sim \mu^n$ be a finite sample, $X^s = M^s \cap X_n$ be the set of source nods, and $\rho$ be a scaling constant. Let $v_n^{*}$ be the induced grounded metric voltage over $(X_n, W, \rho)$. Then, for  all $x_i \in X_n$, we have
\begin{equation*}
    v_n^*(x_i) = \frac{\sum_{x_j \in X^s} k(x_i, x_j) + \sum_{x_j \in \X_n \setminus X^s} v_n^*(x_j)k(x_i, x_j)}{\rho + \sum_{\xbs_j \in \X_n} k(x_i, x_j)},
\end{equation*}
along with $v^*_n(\x_i) = 1$ for $\x_i\in \X^s$.
\label{corollary:Grounded_metric_voltage_function_expression}
\end{proposition}

\begin{proof}
Follows from Lemma \ref{lemma:Solution_EMV_for_LVE}, by writing each element in the voltage vector explicitly, for a given choice of kernel $k$ as weights.

\end{proof}

Proposition \ref{corollary:Grounded_metric_voltage_function_expression} suggests a natural limit object for $v_n^*$.

\begin{theorem}\label{thm:existence_ground}
Let $\M$ be a metric space, and $\M^s$ denote a measurable set of source vertices. Let $\rho > 0$ be a scaling constant. Then there exists a unique map $v^*: \M \to [0, 1]$ such that $v^*(x) = 1$ for all $x \in \M_1$ and the following holds for all $x \in \M \setminus \M^s$: $$v^*(x) = \frac{\int_{\M^s} k(x,y)d\mu(y) + \int_{M \setminus M^s} v^*(y)k(x,y)d\mu(y)}{\rho + \int_\M k(x,y)d\mu(y)}.$$
\end{theorem}

Finally, to prove convergence, we must extend our solutions, $v_n^*$ over grounded metric resistor graphs to the entire metric space $\M$. 

\begin{definition}\label{defn:extended_solution}
Let $v_n^*$ be the EMV over grounded graph $(X_n, W, \rho)$. For all $x \in \M$, we define the \textbf{extension} of $v_n^*$ as the map $v_n^*: M \to [0, 1]$ satisfying the following: $v_n^*(x) = 1$ if $x \in \M^1$, and otherwise, $$v_n^*(x) = \frac{\sum_{i=1}^n k(x, x_i)v_n^*(x_i)}{\sum_{i=1} k(x, x_i)}.$$ 
\end{definition}

\begin{theorem}\label{thm:convergence_ground}
Fix $\rho$ and $M_1 \subseteq M$. Let $v_n^*$ denote the extension of the EMV over the grounded graph, $(X_n, W, \rho)$, and $v^*$ denote the limit object described in Theorem \ref{thm:existence_ground}.  Then for any $x \in \M$, the sequence $v_1^*(x), v_2^*(x), \dots $ converges to $v^*(x)$ in probability (taken over the randomness of sampling $X_n \sim \mu^n$).
\end{theorem}

\subsection{Bounding the shape of the voltage function}
\label{section:Bounding the shape of the voltage function}We have shown that the grounded metric voltage converges to a non-trivial function in the large sample limit. In this section, we want to gain insights into the shape of this voltage function. Throughout this analysis we assume the radial kernel $k(x, y) = d(||x - y || \leq r)$, and restrict the analysis to the unit sphere $S^{d-1}$. Furthermore, we assume a uniform density and use the Lebesgue measure, $\mu(A) = vol(A)$ to have a distribution to draw from.

We want to build a grounded metric graph on the sphere. Suppose the source region $\M_1$ consists of the density contained in a ball $B(x_s, r_s)$ of radius $r_s$ centered on the source landmark $x_s$, where  $x_s$ is a point on the sphere. 
We can then construct a grounded metric resistor graph as described in Def \ref{definition:grounded_metric_resistor_graph}, with resistance to ground $\rho$. In the previous sections, we showed that a non-trivial voltage function $v^*: S^{d-1} \to [0, 1]$ exists for this setting. For our configuration, we denote this function $\lambda$.
When the dimension $d$ is fixed, this function is determined by three parameters, namely the source radius $r_s$, the kernel radius $r$, and resistance to ground $\rho$. In this setting, we have the following bounds on the shape of the grounded metric voltage function:

\begin{theorem}\label{thm:single_landmark_bounds} Let $z \coloneqq d_M(x_s, x) = \arccos(\langle x_s, x \rangle)$ be the geodesic distance from the source landmark $x_s$ to a point $x$ in the unit sphere. Furthermore, let $\phi(r) \coloneqq d_m(x_c, x)$ be the arch-length (geodesic) between two points that have euclidean distance $r$, where $r$ is also the radius of the kernel $k$ used to construct the grounded metric graph. Let $r_s, r, \rho$ be fixed. It then exists a unique map $\lambda: S^{d-1} \to [0, 1]$ satisfying the following properties:
\begin{equation*}
    \lambda(x) = \frac{\int_{B(0, r_s)} k(x, y)d\mu(y) + \int_{S^{d-1}} k(x,y)\in d(y \in S^{d-1} \setminus B(0, r_s))\lambda(y)d\mu(y)}{\rho + a}
\end{equation*}
for $x \notin B(0, r)$ where $a$ denotes the volume of the ball of radius $r$, $\lambda(x) = 1$ for $x \in B(0, r)$ and for $t\in [1, \infty)$
\begin{center}
\begin{itemize}
\item $\lambda$ is radially symmetric and monotonically decreasing in distance outside of $B(0, r_s)$. In particular, there exists a function $h$ such that $\lambda(x) = h(z)$. 
\item (Upper Bound) For $z \geq 2r$, $  h(z + t \phi(r)) \leq \exp{(-t\ln{(1+2\rho/a)})}$.
\item (Lower Bound) There exists a constant $\Gamma$ s.t. $h(z + t \phi(r/2)) \geq \exp{(-t\ln{((a+\rho)/\Gamma)}}$
\end{itemize}
\end{center}
\end{theorem}

Theorem \ref{thm:single_landmark_bounds} shows that the voltage function $\lambda$ is essentially bounded between two functions that exponentially decay with respect to $z=d_M(x_s, x)$, the geodesic distance from $x$ to the source landmark located at $x_s$. We also note that the result in Theorem \ref{thm:single_landmark_bounds} can easily be translated to a Disk in $\bbR^d$ by replacing the geodesic distance with the euclidean distance. This holds because the symmetry arguments used to derive the result for a sphere are also true for a disk. Corollary \ref{coro:voltage_shape_disk} summarize this result.

\begin{corollary}\label{coro:voltage_shape_disk}
For a Disk in $\bbR^d$ we have the result in Theorem \ref{thm:single_landmark_bounds}, with the geodesic replaced by the euclidean distance, s.t. $z=d(x_s, x)$, and the following bounds on $h$:
\begin{center}
\begin{itemize}
\item (Upper Bound) For $z \geq 2r$, $  h(z + t r) \leq \exp{(-t\ln{(1+2\rho/a)})}$.
\item (Lower Bound) There exists a constant $\Gamma$ s.t. $h(z + t r) \geq \exp{(-2t\ln{((a+\rho)/\Gamma)}}$
\end{itemize}
\end{center}
\end{corollary}


\subsection{Examples}
\label{sect:examples_grounded_resistor_graphs_over_metric_spaces}

We demonstrate the convergence of the grounded metric graph voltage on several basic examples in Figure \ref{fig:voltageGrounded}.  

\begin{figure}[h!]
\centering
\begin{tabular}{ccc}
\includegraphics[width=.3\textwidth,height=.2\textwidth]{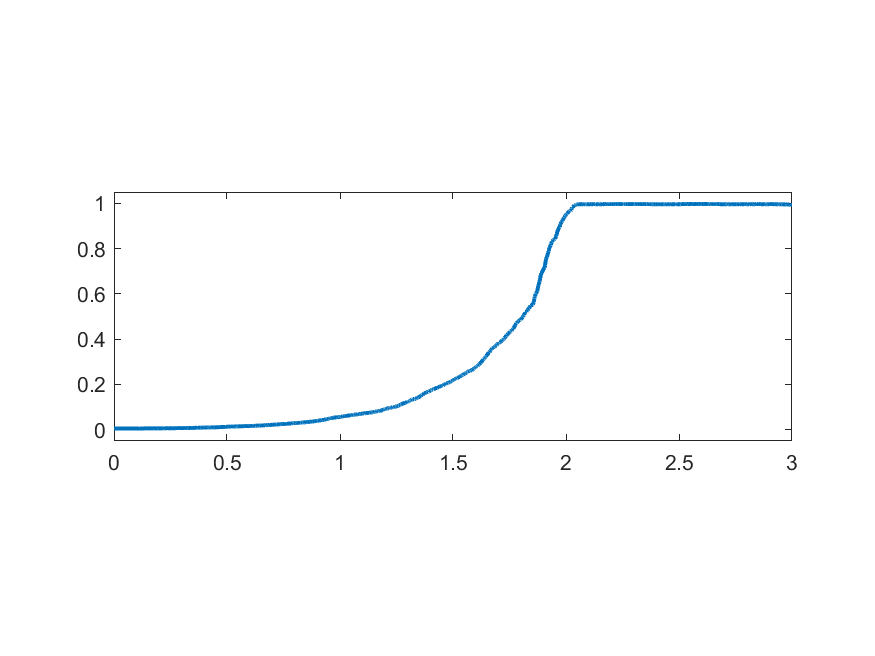} \vspace{-.75em}& \hspace{-3em}
\includegraphics[width=.3\textwidth,height=.2\textwidth]{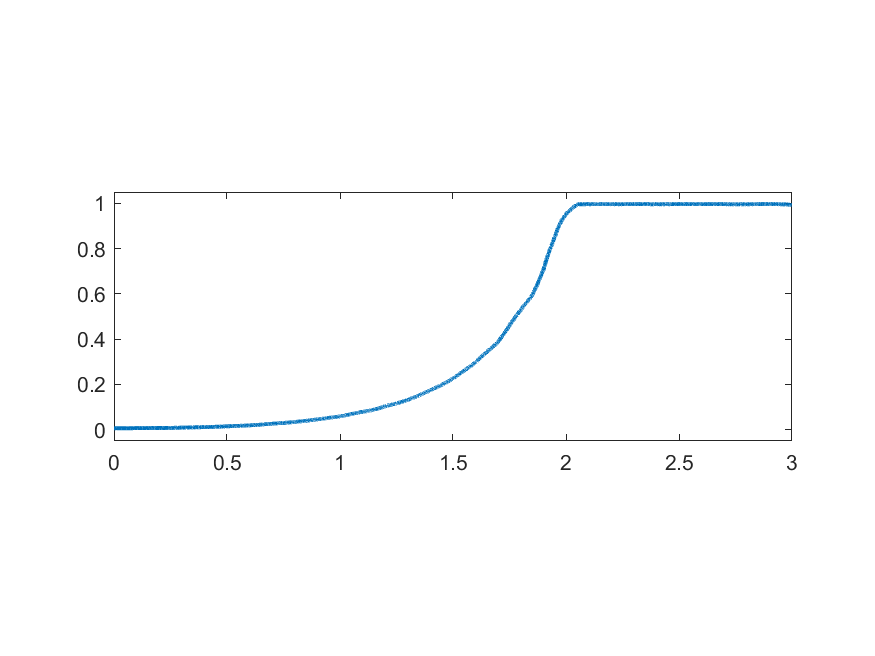} \vspace{-.75em} & \hspace{-3em}
\includegraphics[width=.3\textwidth,height=.2\textwidth]{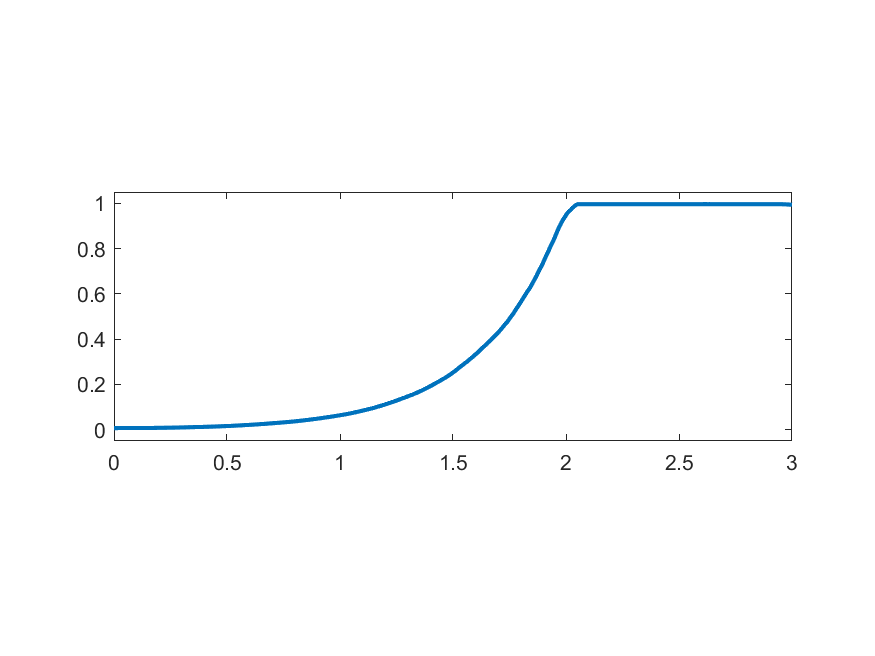} \vspace{-.75em} \\
\includegraphics[width=.3\textwidth,height=.2\textwidth]{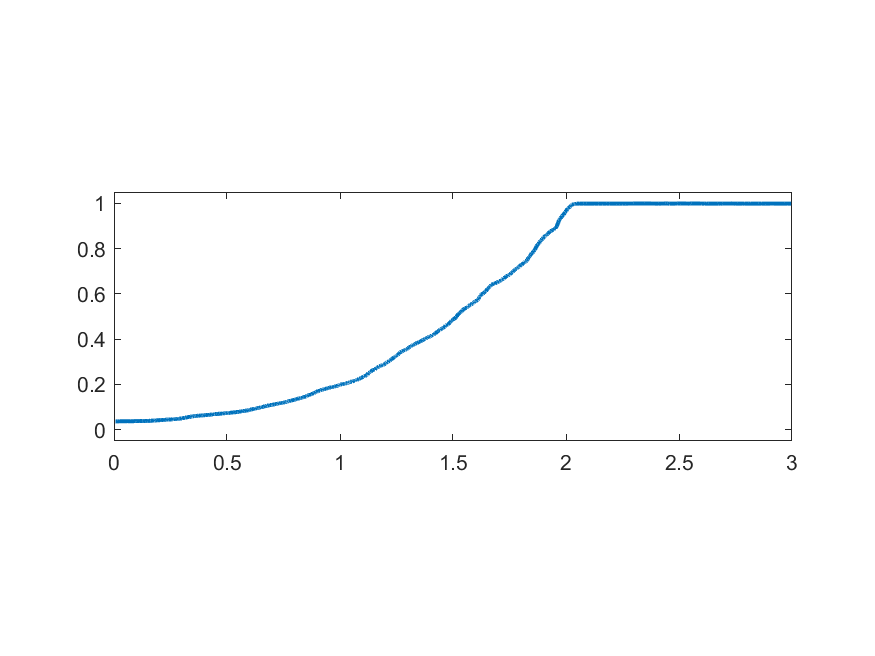} \vspace{-.5em}& \hspace{-3em}
\includegraphics[width=.3\textwidth,height=.2\textwidth]{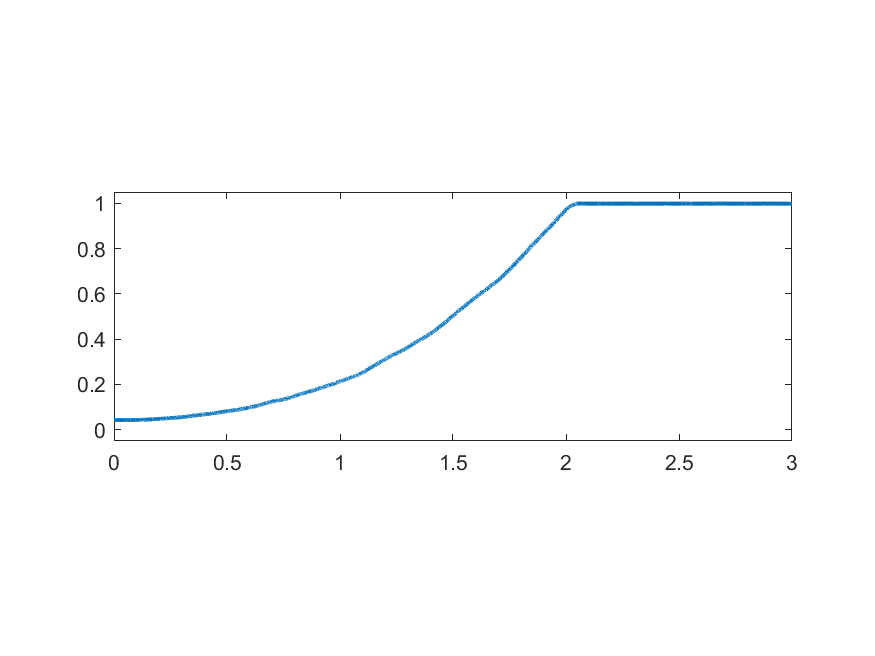} \vspace{-.5em} & \hspace{-3em}
\includegraphics[width=.3\textwidth,height=.2\textwidth]{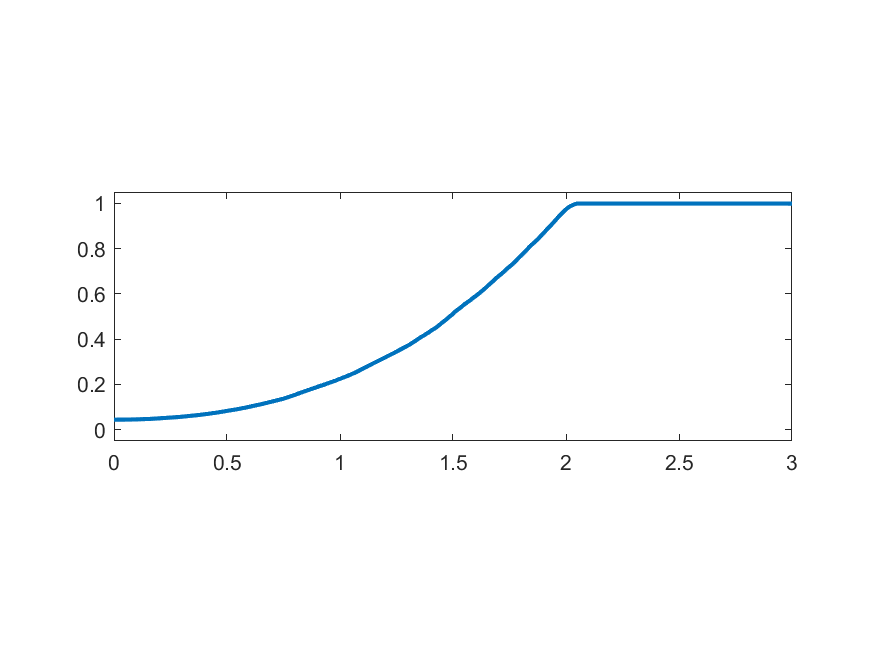} \vspace{-.5em} \\
\includegraphics[width=.3\textwidth]{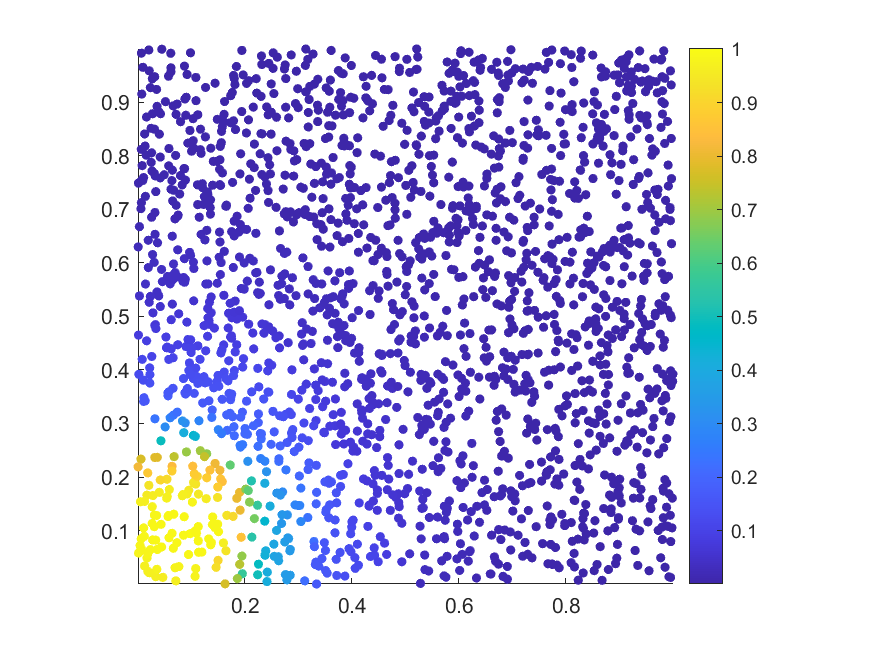} & \hspace{-3em}
\includegraphics[width=.3\textwidth]{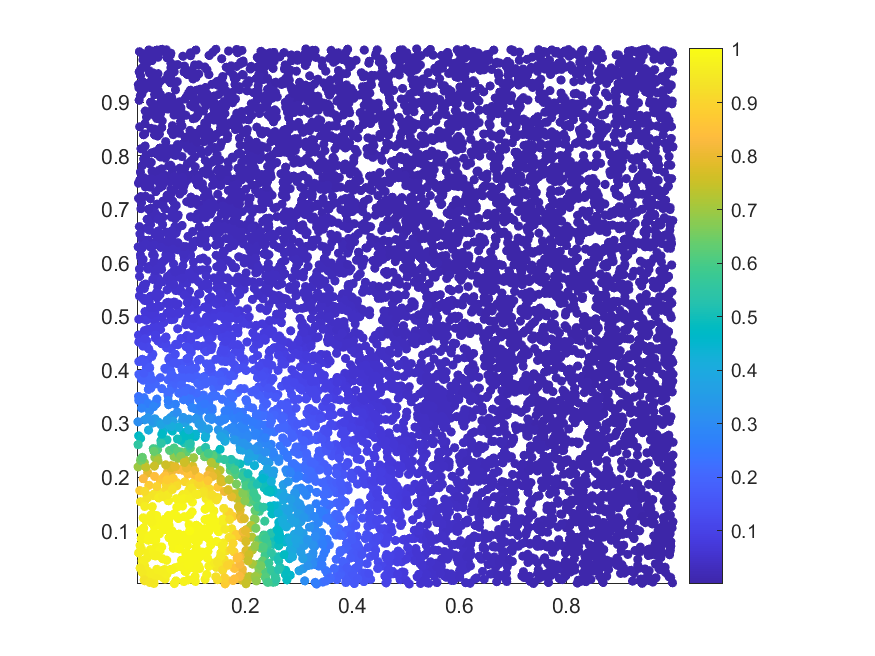} & \hspace{-3em}
\includegraphics[width=.3\textwidth]{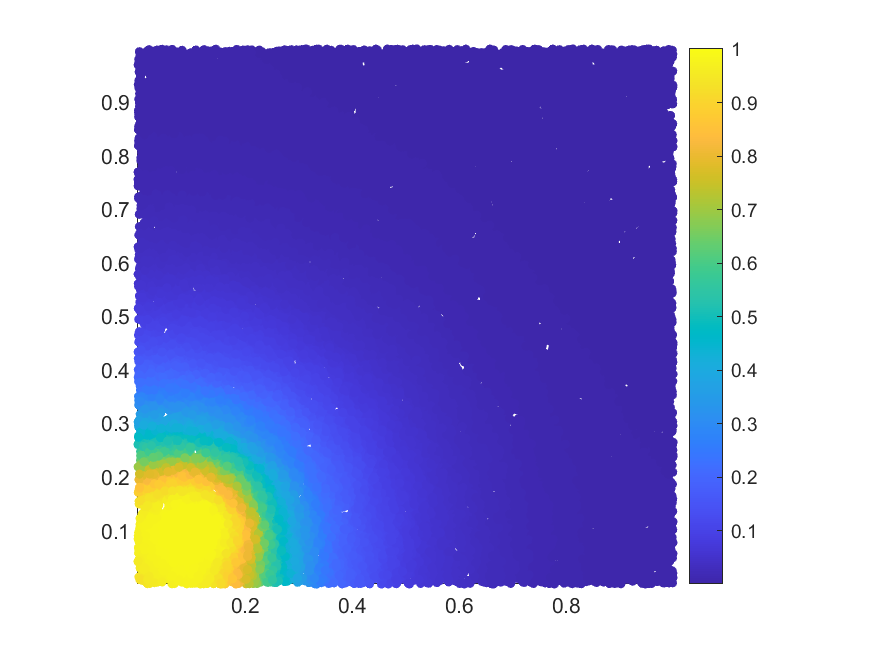}
\end{tabular}
\caption{Grounded metric graph voltage.  {\bf Top:} $\X_1$ on $[2,3]$ with large $\rho$ to encourage fast decay.  {\bf Middle:} $\X_1$ on $[2,3]$ with small $\rho$ to encourage slow decay.  {\bf Bottom:} $\X_1$ centered at $(.1,.1)$ with radius $0.1$ with large $\rho$.  {\bf Left:} $n=2^{11}$, {\bf Center:} $n=2^{13}$, {\bf Right:} $n=2^{15}$}\label{fig:voltageGrounded}
\end{figure}

The first is on the 1D line $[0,3]$, where $\X_1$ is a source of radius $0.5$ on the interval $[2,3]$, and the sink $\X_0$ is of radius $0.5$ on the interval $[0,1]$.  The second example is on the 2D unit square, with the source and sink on opposite corners of the square.  In both cases, the Laplacian is formed with radius $r=0.05$, and the number of points range from $2^{11}, 2^{13}, 2^{15}$.  We also vary $\rho$ to demonstrate the effect of the ground radius on the speed of decay of the voltage. In all cases, the voltage function is strictly non-increasing away from the source.


\section{Embedding using localized voltage functions}
\label{section:Embedding using localized voltage functions}
In this section, we show how the grounded metric voltage function from Def. \ref{def:grounded_metric_voltage_function} can be used as an embedding tool for various manifolds. The experiments in section \ref{sect:examples_grounded_resistor_graphs_over_metric_spaces} along with Thm. \ref{thm:single_landmark_bounds} demonstrate, for a unit sphere $S^{d-1}$ and a disk in $\bbR^d$, a consistent decay of the voltage solution away from the source. The advantage of this behavior is that the voltage solution gives information about how far other points in $\M$ are from the source vertex; points with high voltages must be close, whereas points with low voltages must be far. Our idea is that the decay property can be used to construct an embedding using a selection of independent voltage functions generated by sources distributed across the manifold.

Our first step is to show theoretically that an embedding exists for the unit sphere $S^{d-1}$. With a basis in the discussion in
\cite{li2017efficient} on spherical principal components analysis, we argue that there exists a large family of manifolds that can be approximated locally by sphere segments, which makes this result apply to a variety of cases beyond $S^{d-1}$. We illustrate the method's applicability on several manifolds, including a sphere and the MNIST data set. 







\subsection{Embedding of a unit sphere $S^{d-1}$}
We show theoretically that we can embed the unit sphere, $S^{d-1}$, using ground voltage vectors from $d$ voltage sources. We will use a disk kernel $k$ with bandwidth $0 < r < 1$, and we will also take source radius $r_s = r$. We begin with the notion of an $\epsilon$-injective embedding. We will also assume that our data distribution over $S^{d-1}$ is the uniform distribution.

\begin{definition}
Let $f: X \to Z$ be a map between metric spaces $X, Z$. For $\epsilon > 0$, we say that $f$ is \textbf{$s$-injective} if $$d(x, x') > \epsilon \implies f(x) \neq f(x').$$
\end{definition}

Our goal will be to find an $\epsilon$-injective embedding of $S^{d-1}$. To do so, we consider the standard embedding of $S^{d-1}$ in $\reals^d$, and let $||x- x'||$ denote the $\ell_2$ distance between points $x, x'$. Note that because the distance metric completely determines our voltage functions, it follows that any result for this setting will carry over to any isometric embedding of $S^{d-1}$.

Next, we characterize voltage functions over $S^{d-1}$. We let $\angle(x, x')$ denote the angle between them on the sphere. That is $\angle(x, x') = \arccos \langle x,x' \rangle.$ We also let $r' = \angle(x, x'): ||x - x'|| = r,$ denote the angle between two points with an $\ell_2$ distance of $r$. 

\begin{theorem}
Let $x_0 \in S^{d-1}$ be an arbitrary point, and let $v_0$ be the grounded voltage function centered at $x_0$ with ground resistance $\rho$, source radius $r_s = r$, and kernel function $k$ being the disk kernel with bandwidth $r$. Then there exists a function $f: [0, \pi] \to 1$ with the following properties. 

1. $f$ fully determines $v_0$. That is, $v_0(x) = f(\angle(x_0, x))$ for all $x \in S^{d-1}$).
 
2. $f$ satisfies $f(\theta) > f(\theta')$, for all $\theta' > \theta + r$ 

\end{theorem}

\begin{proof}
Property 1 immediately holds due to the rotational symmetry of the sphere about $x_0$. Furthermore, by the radial symmetry of the sphere, the same function $f$ must suffice for all $x_0 \in S^{d-1}$.

To show property 2, we first show that $f$ is weakly monotonic, that is $f(\theta) \geq f(\theta')$ if $\theta \leq \theta'$. To do so, let $A_*, b_*$ be the operator and function from Definition \ref{defn:affine} such that $v_0 = A_*v_0 + b_*$. As we showed earlier, we have that $$v_0 = \sum_{i=0}^\infty A_*^ib_0.$$ Thus, it suffices to show that all of the partial sums of this series are weakly monotonic. We do so through induction. 

The base case is trivial as $b_0$ is clearly weakly monotonic with respect to the geodesic distance. For the inductive step, suppose that $v_0^n = \sum{i=0}^n A_*^ib_0$ is weakly monotonic. Fix any $\theta \leq \theta'$ and let $x, x' \in S^{d-1}$ be points such that $\angle(x_0, x) = \theta, \angle(x_0, x') = \theta'$, and $x_0, x, x'$ all lie on a great circle (that is, along a geodesic).

Let $B$ denote the disk of radius $r$ centered at $x$ intersected with $S^{d-1}$, and $B'$ denote the analogous disk for $x'$. The key observation is that $B'$ is precisely the reflection of $B$ across the perpendicular bisector of $x, x'$ in $S^{d-1}$. Furthermore, this reflection is clearly an isometry and preserves the uniform measure $\mu$ over $S^{d-1}$. We let $\tau$ denote this reflection. Finally, observe that for all $y \in B \setminus (B \cap B')$, $$\angle(x_0, y) \leq \angle(x_0, \tau(y)).$$ This holds even in the extreme case where $x'$ is the antipodal point to $x_0$. 

We now substitute $\tau$ into the equation $v_0^{n+1} = A_*v_0^n + b_*$. Doing so, we have
\begin{equation*}
\begin{split}
v_0^{n+1}(x') &= A_*v_0^n(x') + b_*(x') = \frac{\int_{B'} k(x', y)v_0^n(y)d\mu(y)}{\rho + \int_{B'} k(x', y)d\mu(y)} + b_*(x') \\
&= \frac{\int_{B} k(x', y)v_0^n(\tau(y))d\mu(y)}{\rho + \int_{B'} k(x', y)d\mu(y)} + b_*(x') \\
&\leq \frac{\int_{B} k(x', y)v_0^n(y)d\mu(y)}{\rho + \int_{B'} k(x', y)d\mu(y)} + b_*(x') = v_0^{n+1}(x),
\end{split}
\end{equation*}
with the inequalities holding from our observations about $\tau$.

Finally, having shown that $v_0$ is weakly monotonic, we now turn to Property 2. Fix $\theta, \theta'$ and let $x, x'$ be such that $\angle(x_0, x) = \theta$, $\angle(x_0, x') = \theta'$. The key observation is that for all $y \in B(x', r)$, $\angle(x_0, y) > \theta$. This is from simple geometry as $\theta' > \theta + r'$. Thus, it consequently follows that $v_0(y) \leq v_0(x)$ for all such $y$. Substituting this, we have that 
\begin{equation*}
\begin{split}
v_0(x') &= A_*v_0(x') + b_*(x') = \frac{\int_{B'} k(x', y)v_0(y)d\mu(y)}{\rho + \int_{B'} k(x', y)d\mu(y)} + b_*(x') \\
&\leq \frac{\int_{B'} k(x', y)v_0(x)d\mu(y)}{\rho + \int_{B'} k(x', y)d\mu(y)} + b_*(x') \\
&\leq v_0(x) \frac{\int_{B'}k(x', y)d\mu(y)}{\rho + \int_{B'}k(x', y)d\mu(y)} + b_*(x') \\
\end{split}
\end{equation*}
with the last inequality holding since $b_*(x') = 0$ as $x'$ cannot be inside the source region as $\angle(x_0, x') > r'$.
\end{proof}

We now show how to obtain an $\epsilon$-injective embedding of $S^{d-1}$.

\begin{theorem}
Let $e_1, \dots, e_d$ denote the standard normal basis of $\reals^{d-1}$, and associate them as voltage sources on $S^{d-1}$. Let $v_1, \dots, v_d$ denote their respective voltage functions using a disk kernel of radius $r$ (and a source radius of $r$). Then the map $$x \mapsto (v_1(x), \dots, v_d(x))$$ is an $r'\sqrt{d}$-injective map, where $r' \in [0, \pi]$ denote the angle subtending a chord of length $r$ on the unit circle. 
\end{theorem}

\begin{proof}
Let $x, x'$ be arbitrary points on $S^{d-1}$ with $||x - x'|| > r\sqrt{d}$. Let $\theta_i$ denote the geodesic angle distance of $x$ from $e_i$, and $\theta_i'$ be the same for $x'$. It follows that $$\sum_{i=1}^d (\cos \theta_i - \cos \theta_i')^2 > r^2d.$$ Thus, for some $i$, we must have $(\cos \theta_i - \cos \theta_i')^2 > r^2$. WLOG, this holds for $i = 1$. Applying this, we have 
\begin{equation*}
\begin{split}
r'^2 &< (\cos \theta_1 - \cos \theta_1')^2 = (2\sin \frac{\theta_1 + \theta_1'}{2}\sin\frac{\theta_1' - \theta_1}{2})^2 \\
&\leq 4\sin^2 \frac{\theta_1' - \theta_1}{2} \leq (\theta_1' - \theta_1)^2.
\end{split}
\end{equation*}
Thus $\theta_1' - \theta_1 > r'$ which implies $v_1(x') \neq v_1(x)$, as desired. 
\end{proof}

\subsection{Examples}
To support our theoretical results, we show numerically how the grounded metric voltage function can be utilized to embed the unit sphere. Furthermore, we illustrate our method on two real-world data sets, namely the Frey faces data-set \cite[Accessed: 2022-09-30]{dataset_webpage} and MNIST \cite{deng2012mnist}. 

In the experiments, we build an embedding by computing $m$ independent voltage functions $v^{(i)}_n$ from $m$ different sources $\widetilde{x}_i$ (landmarks), selected randomly from the manifold. From these voltages we then construct an $m$ dimensional embedding $Z = (v^{(1)}_n, \dots, v^{(m)}_n) \in \bbR^{n \times m}$. Since $m>3$ for our experiments, the embeddings can not be visualized directly. Because of this, we create a projection of the embedding into $d\leq 3$ dimensions using $X_{d} = U_d\Lambda_d \in \bbR^{n\times d}$, which corresponds to a multi-dimensional scaling embedding (MDS) \cite{dokmanic2015euclidean}. Here $Z_s = U \Lambda V^\top$ is a centering of $Z$ and $U_d, \Lambda_d$ are the $d$ leading eigenvectors and eigenvalues.

\paragraph{Unit sphere experiment}
We consider the embedding of the two first quadrants of the unit sphere $S^3$. Using $n=2^{13}$ points sampled i.i.d. from this sphere segment we calculate the voltage for $m=\{3, 5, 7, 9\}$ sources. In Fig. \ref{fig:embedding180sphere} we show the results from running these experiments. We see that increasing the number of sources gives an increasingly better embedding, which demonstrates the robustness of this algorithm w.r.t. choosing more sources than the intrinsic dimension. 

\begin{figure}[htb!]
    \centering
    \subfloat[\label{subfig:nlm3Ref}\protect\centering Sphere segment 3 sources ]{\includegraphics[width=0.24\textwidth]{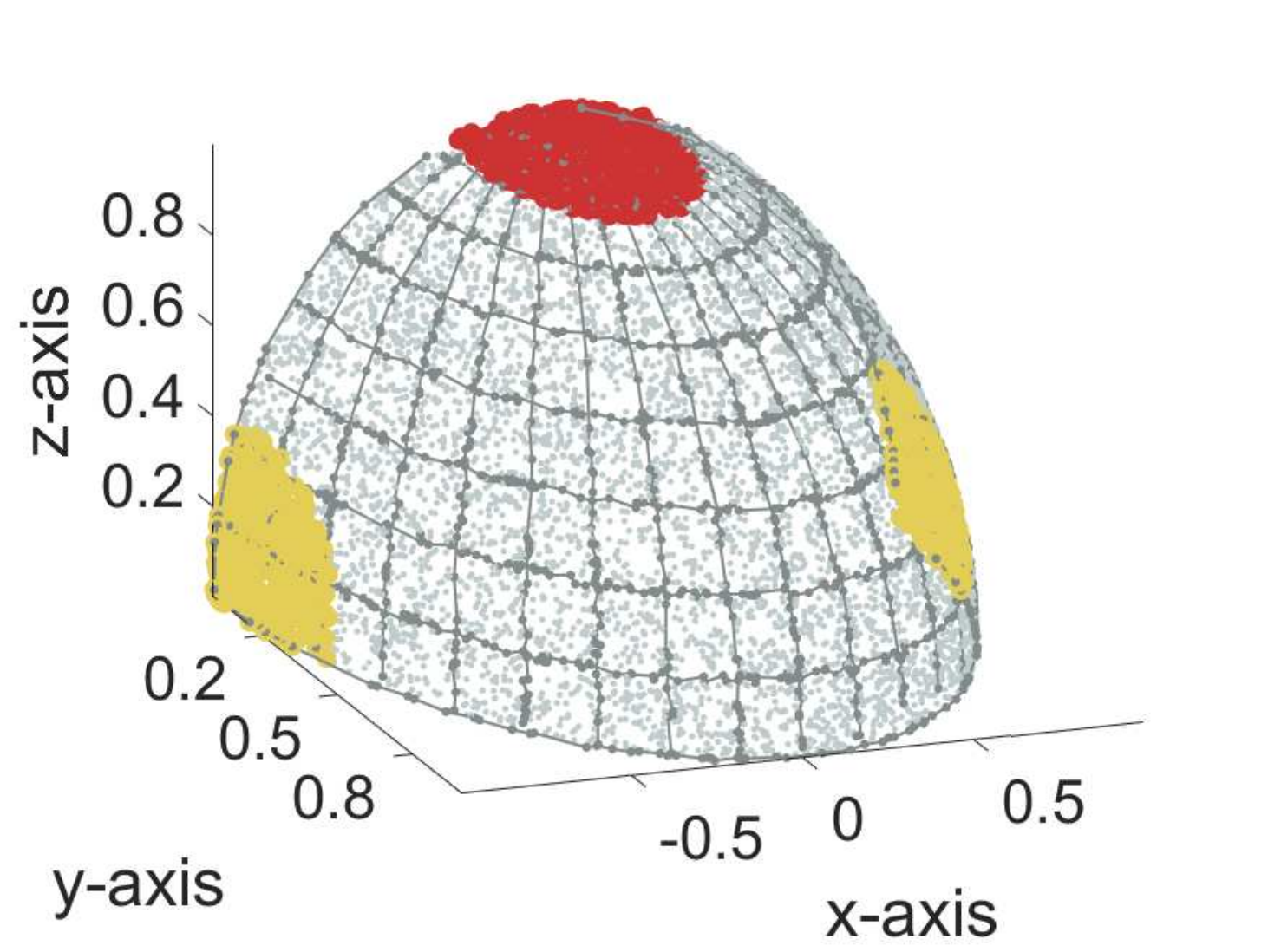} }
        \hfill
    \subfloat[\label{subfig:nlm5Ref}\protect\centering Sphere segment 5 sources ]{{\includegraphics[width=0.24\textwidth]{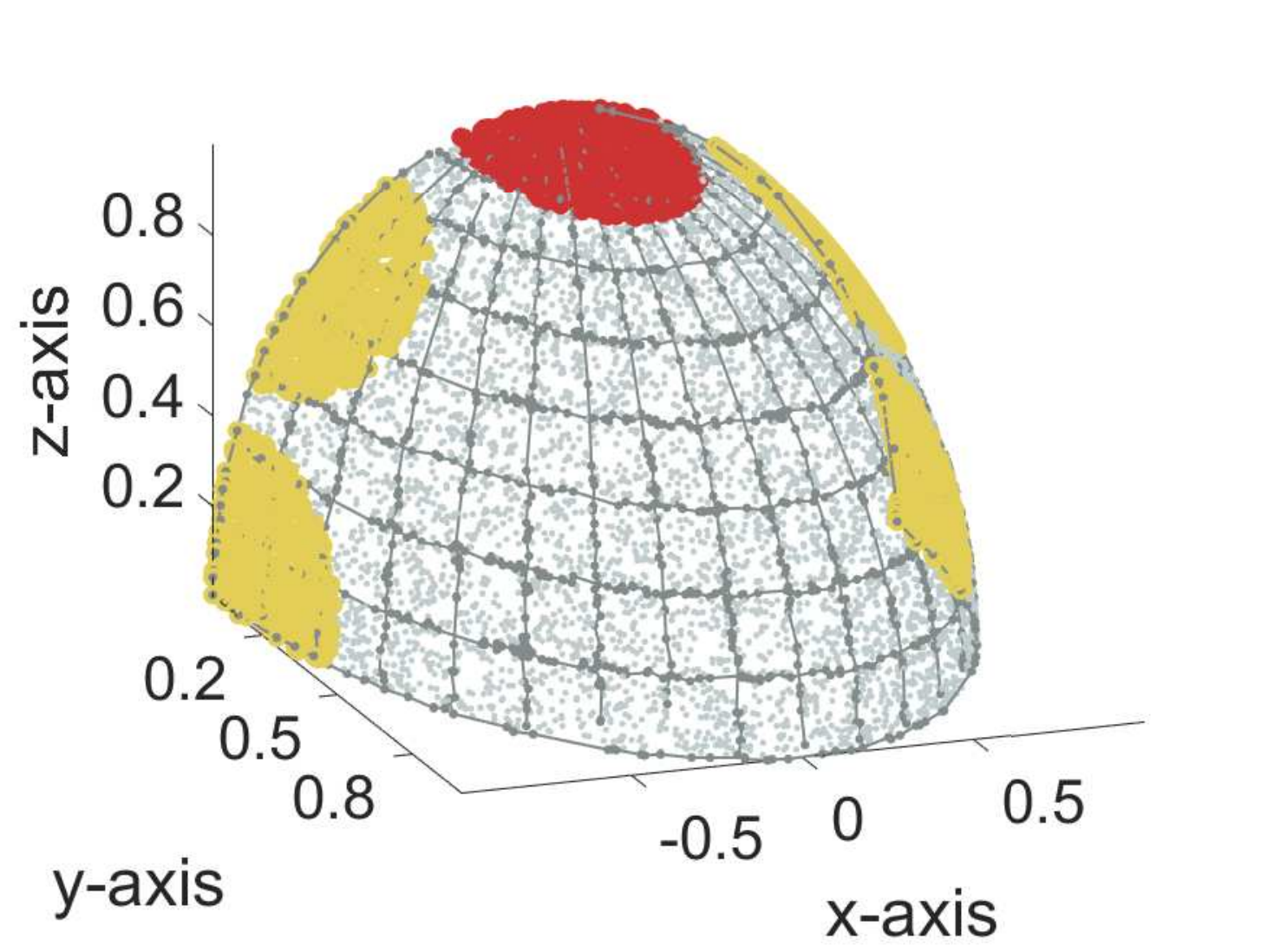} }}%
     \hfill
    \subfloat[\label{subfig:nlm7Ref}\protect\centering Sphere segment 7 sources ]{{\includegraphics[width=0.24\textwidth]{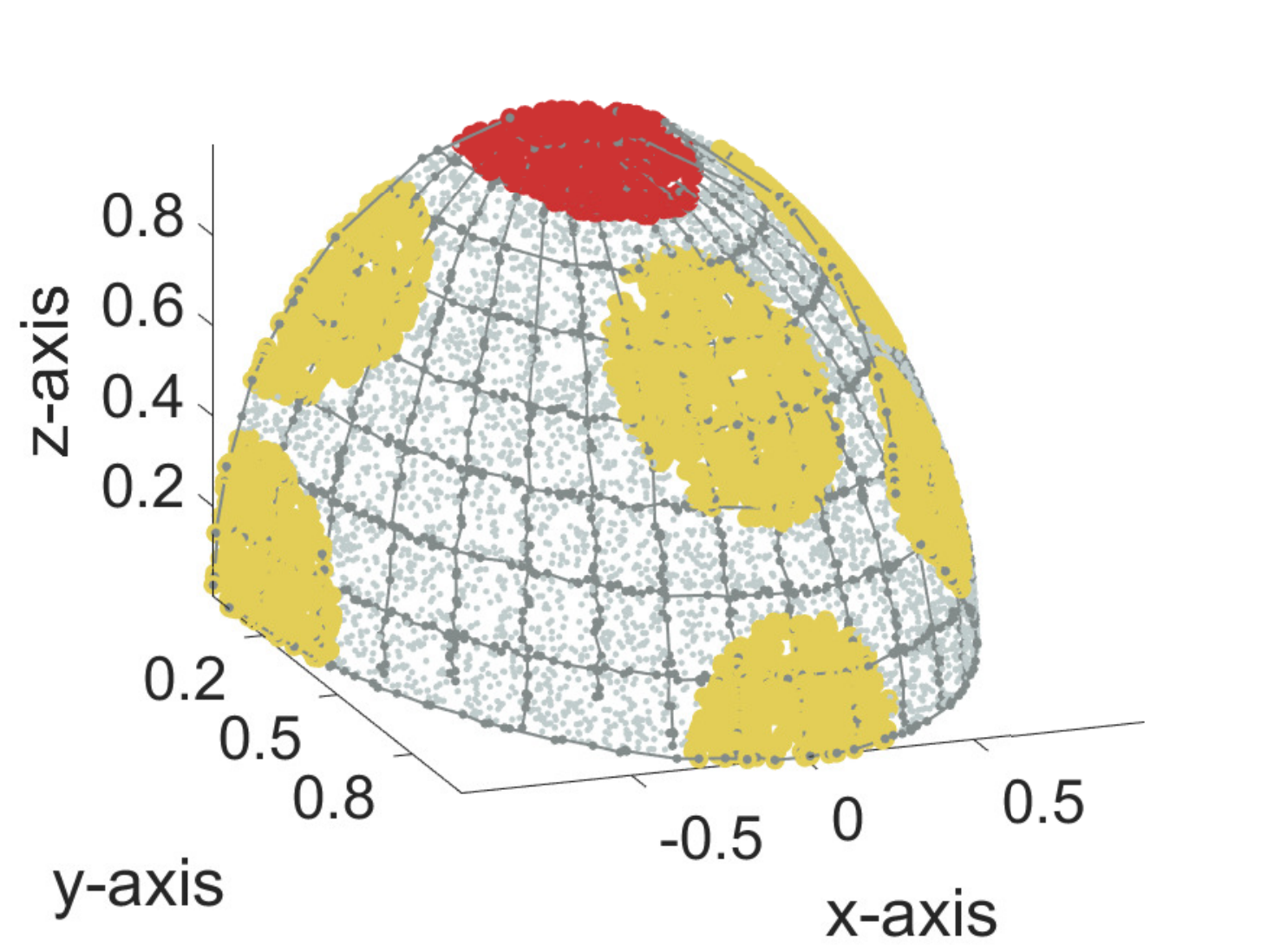} }}%
         \hfill
    \subfloat[\label{subfig:nlm9Ref}\protect\centering Sphere segment 9 sources ]{{\includegraphics[width=0.24\textwidth]{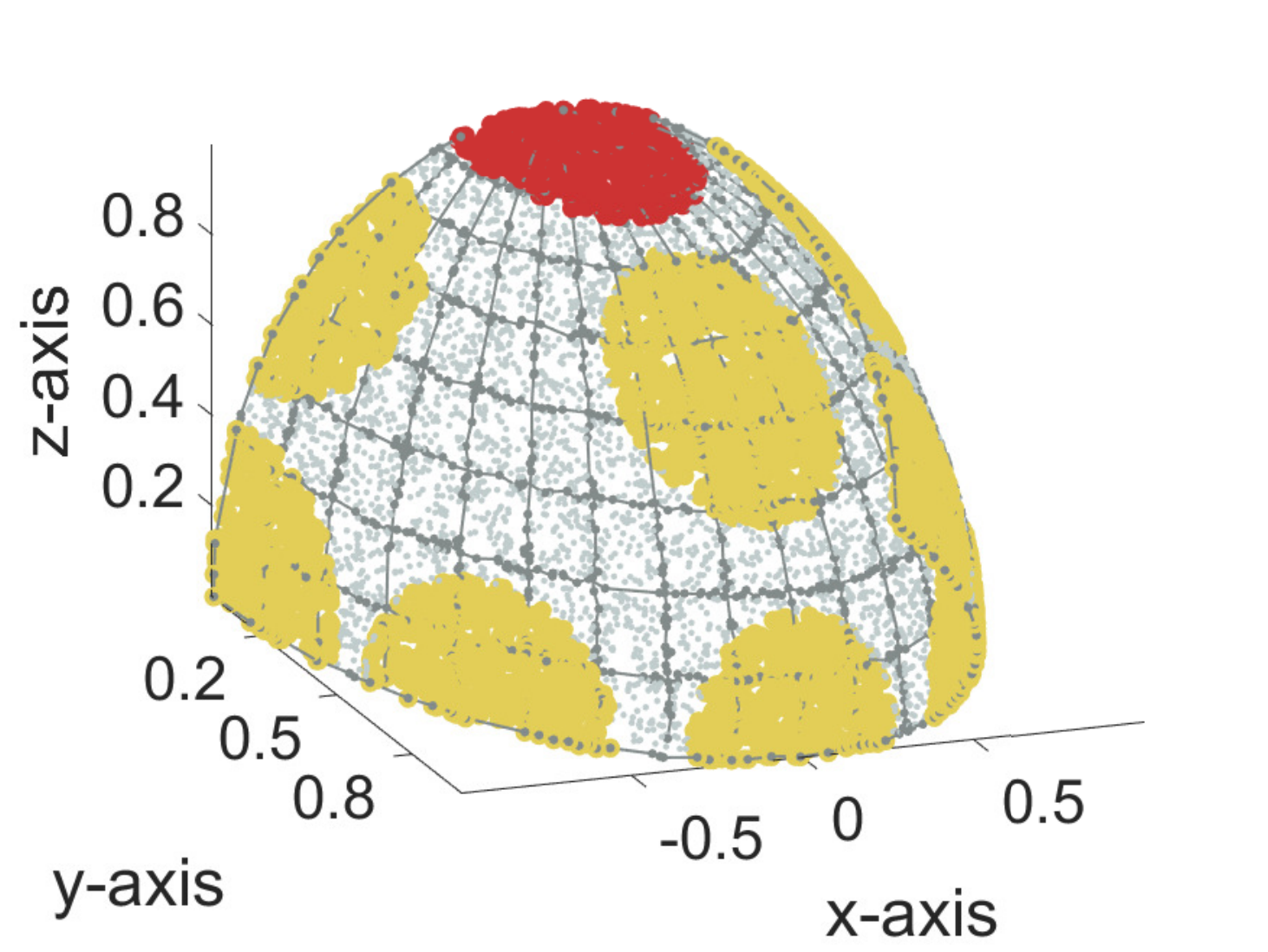} }}%
        \hfill
    \subfloat[\label{subfig:nlm3MDSemb}\protect\centering Embedding 3 sources ]{{\includegraphics[width=0.24\textwidth]{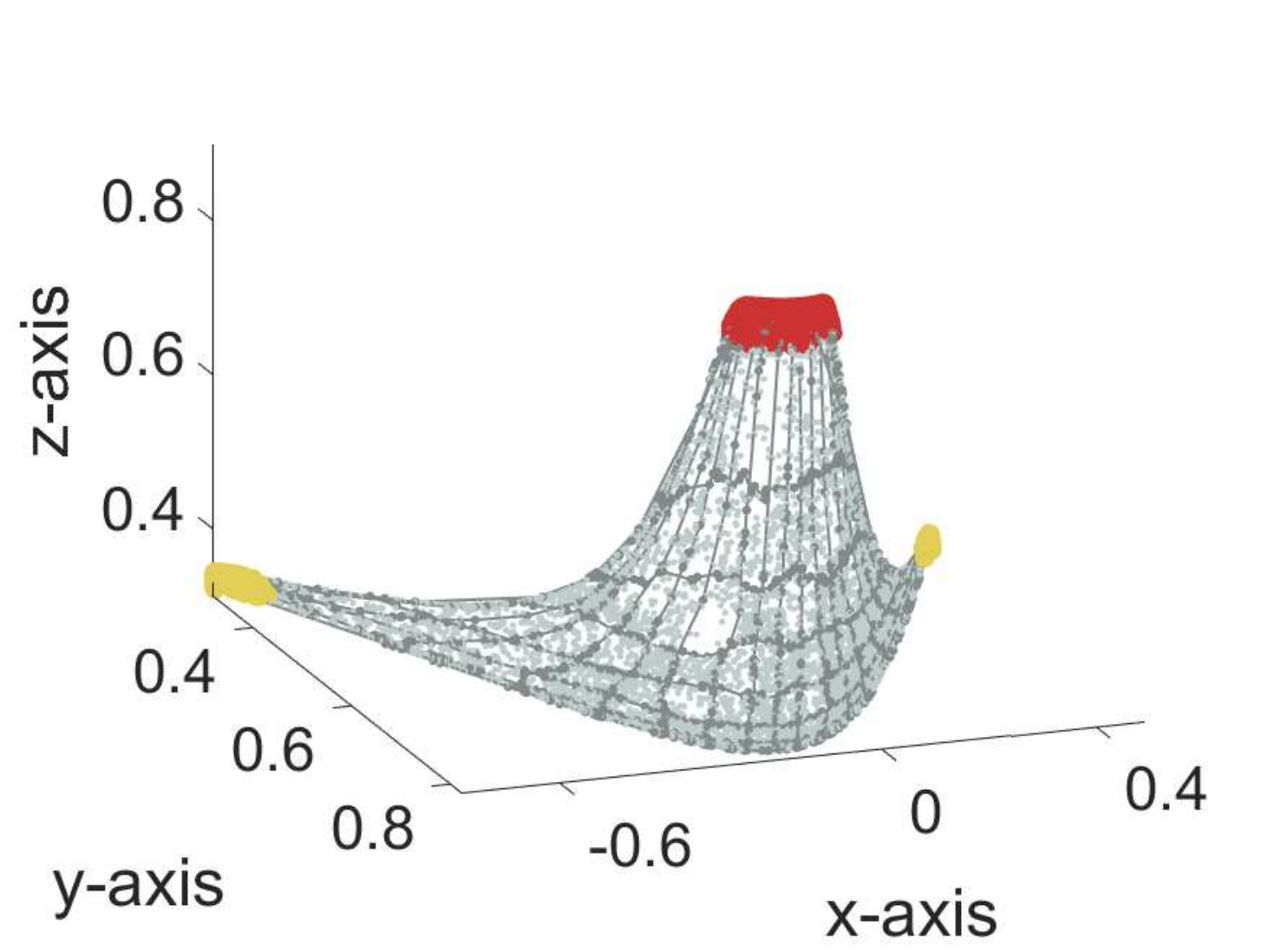} }}%
         \hfill
    \subfloat[\label{subfig:nlm5MDSemb}\protect\centering Embedding 5 sources ]{{\includegraphics[width=0.24\textwidth]{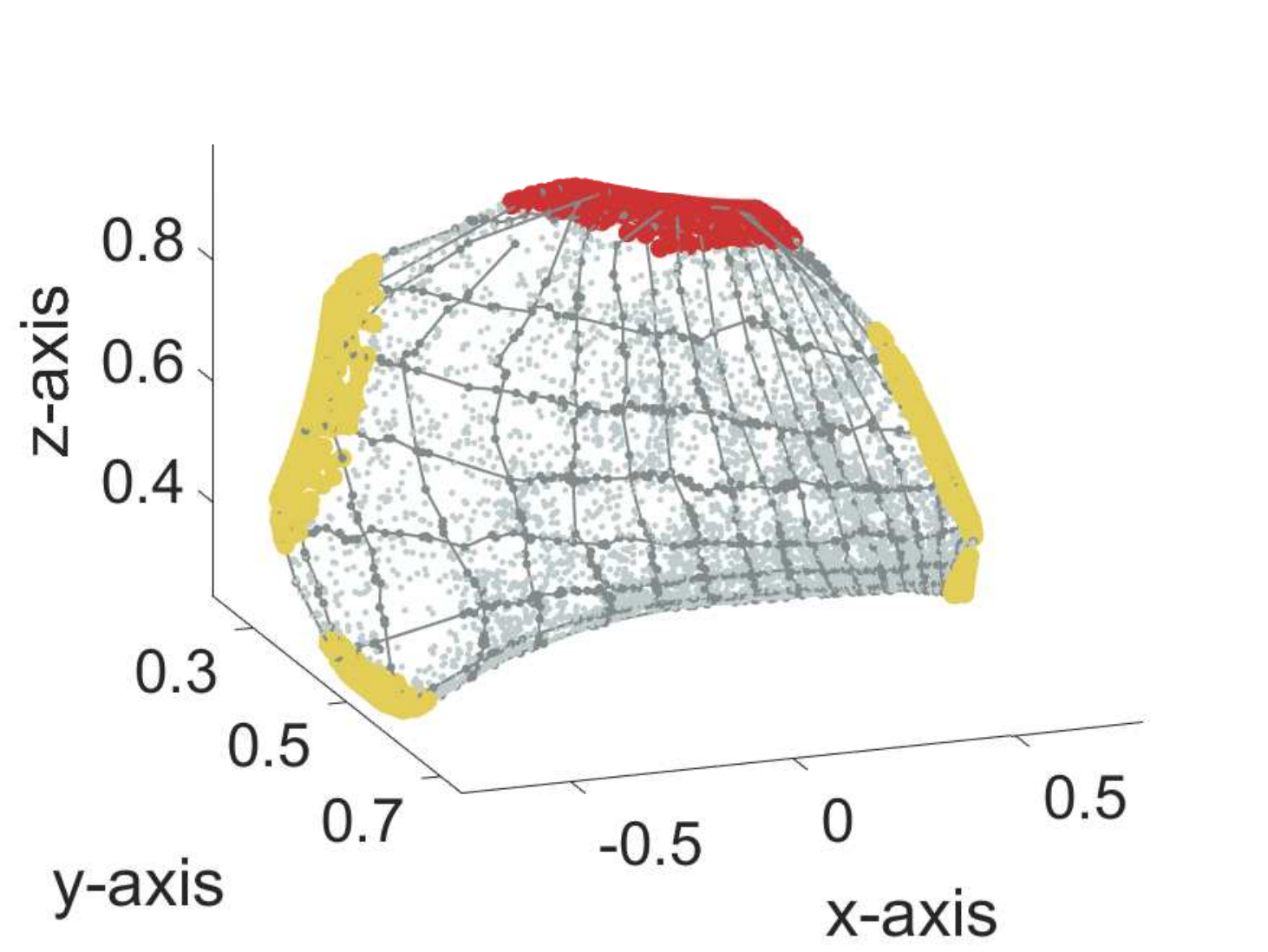} }}%
         \hfill
    \subfloat[\label{subfig:nlm7MDSemb}\protect\centering Embedding 7 sources ]{{\includegraphics[width=0.24\textwidth]{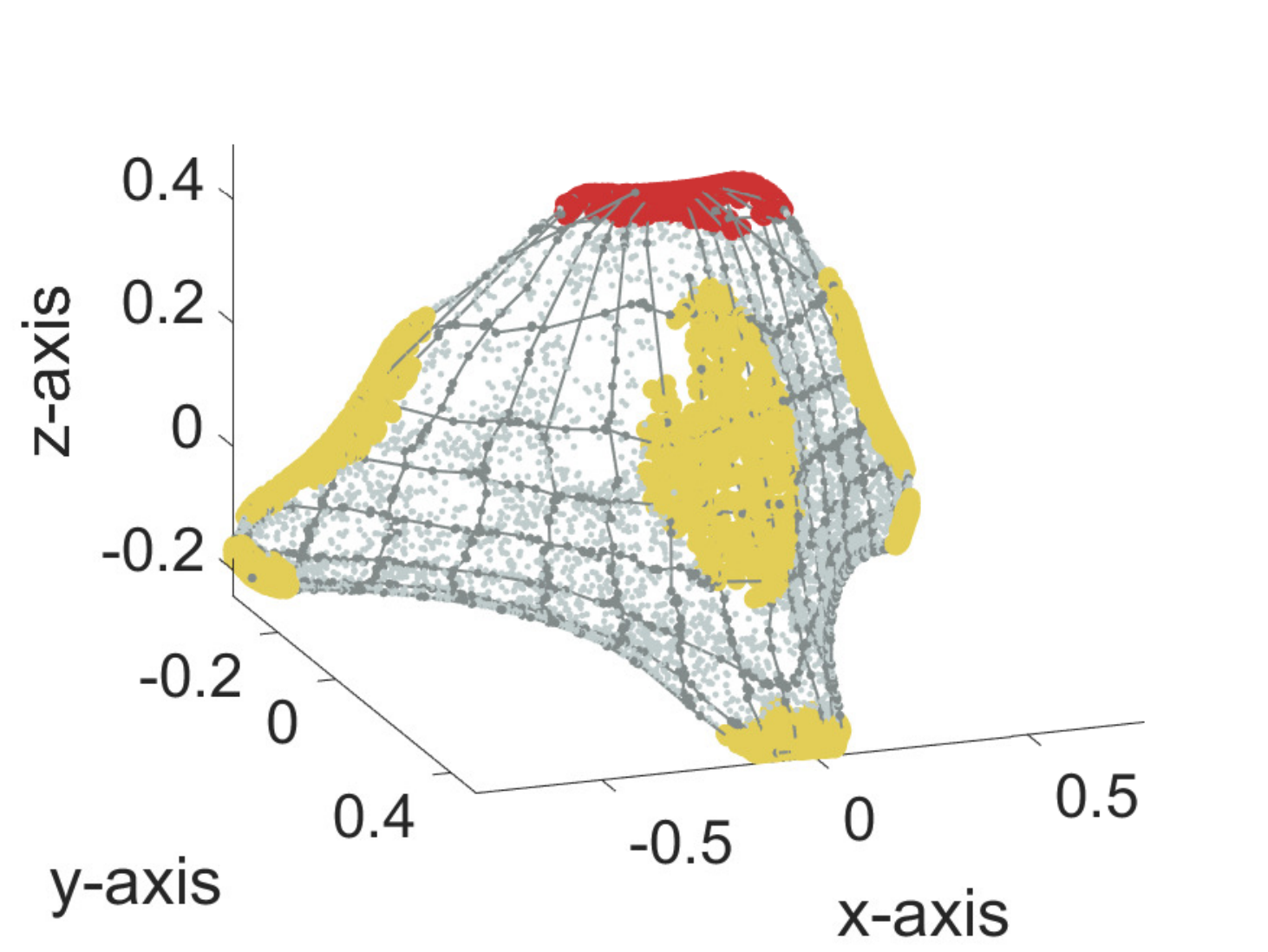} }}%
             \hfill
    \subfloat[\label{subfig:nlm9MDSemb}\protect\centering Embedding 9 sources ]{{\includegraphics[width=0.23\textwidth]{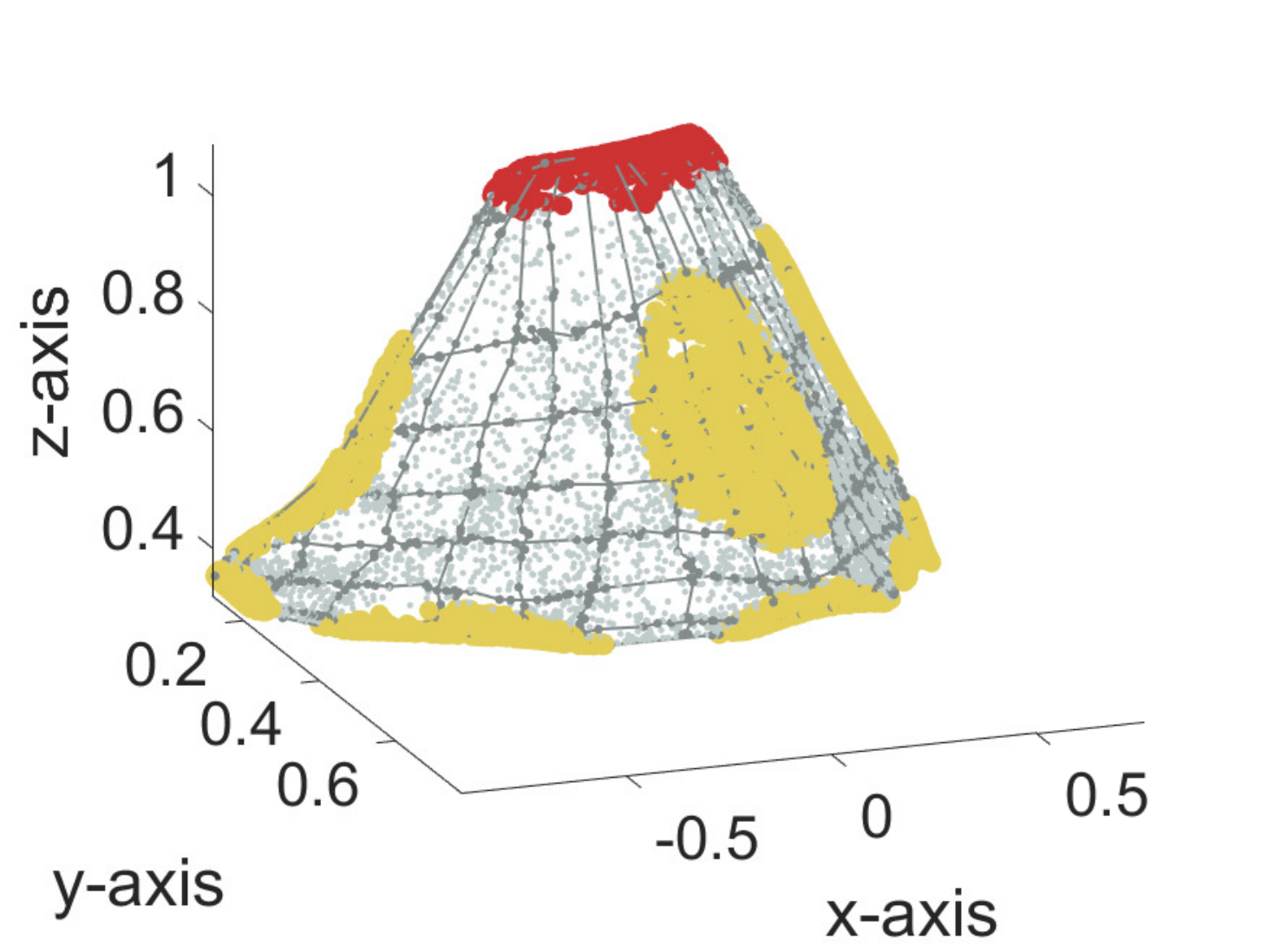} }}%
    \caption{Illustration of the embedding of the two first quadrants of a sphere using localized voltage sources. The yellow and red colored regions correspond to the sources. The upper row is the true manifold and the lower row is the recreation using $m=\{3, 5, 7, 9\}$ sources. In order to align the embedding with the initial data we used orthogonal Procrustes analysis \cite{dokmanic2015euclidean}.} 
    \label{fig:embedding180sphere}%
\end{figure}

\paragraph{Frey-face experiment}
The Frey face data set is taken from \cite[Accessed: 2022-09-30]{dataset_webpage} and consists of $1965$ images of size $20 \times 28$ of Brendan Frey's face. We generate two embeddings using respectively $m=\{3, 10\}$ landmark sources selected randomly from the images. Fig. \ref{fig:freyface_embedding} illustrates a 2-dimensional projection of these embeddings. From the figure we see that the images are clustered based on facial expressions. In particular, the images tend towards landmarks with a similar expression and change gradually as you move away. Comparing Fig. \ref{subfig:ff_embedding_3lm} and Fig. \ref{subfig:ff_embedding_10lm}, we see that additional landmarks open up local clusters by spreading the images out in between. 

\begin{figure}[htb!]
    \centering
    \subfloat[\label{subfig:ff_embedding_3lm}\protect\centering Embedding with 3 landmarks ]{{\includegraphics[width=0.44\textwidth]{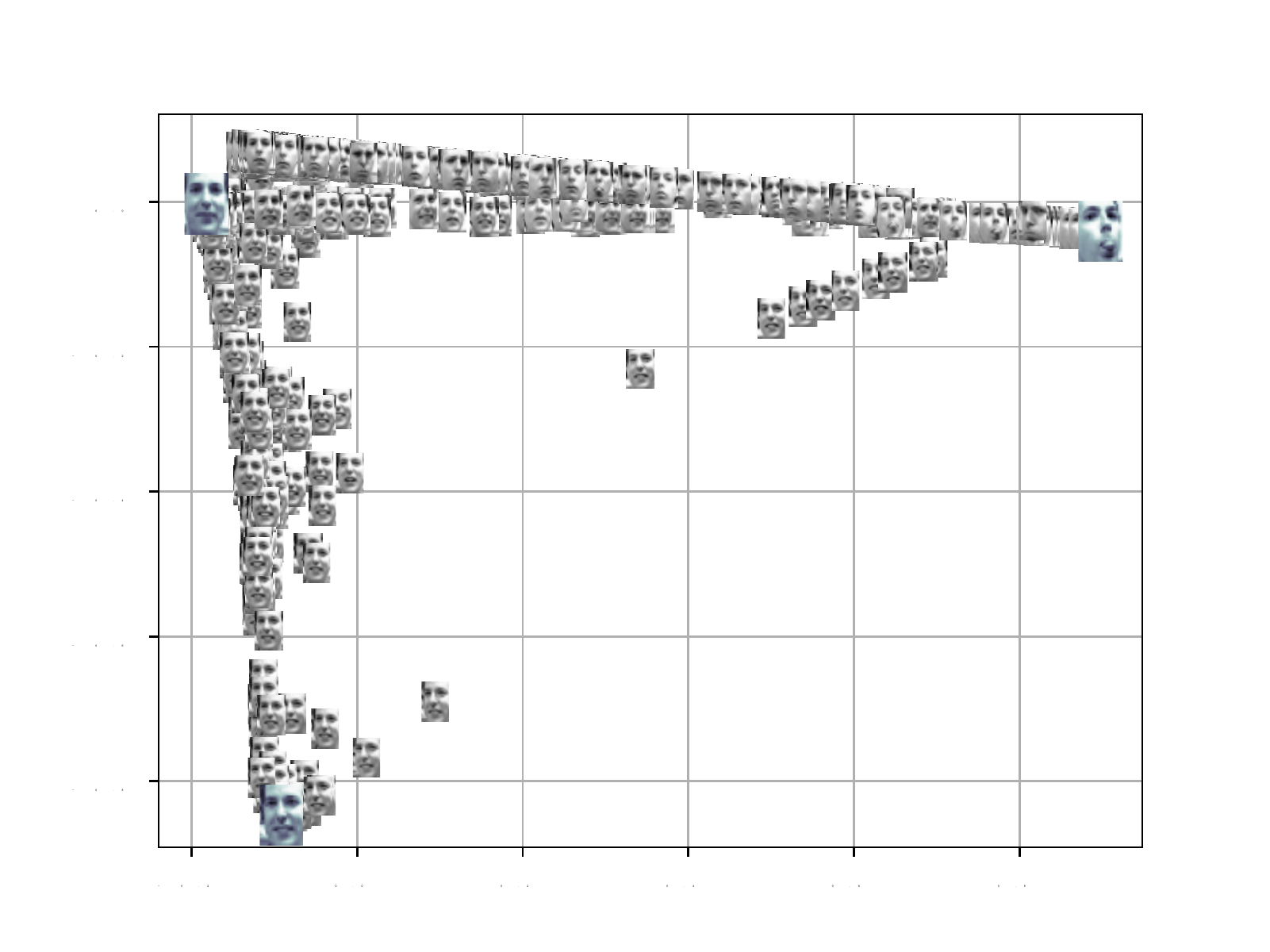} }}%
        \hfill    \subfloat[\label{subfig:ff_embedding_10lm}\protect\centering Embedding with 10 landmarks ]{{\includegraphics[width=0.44\textwidth]{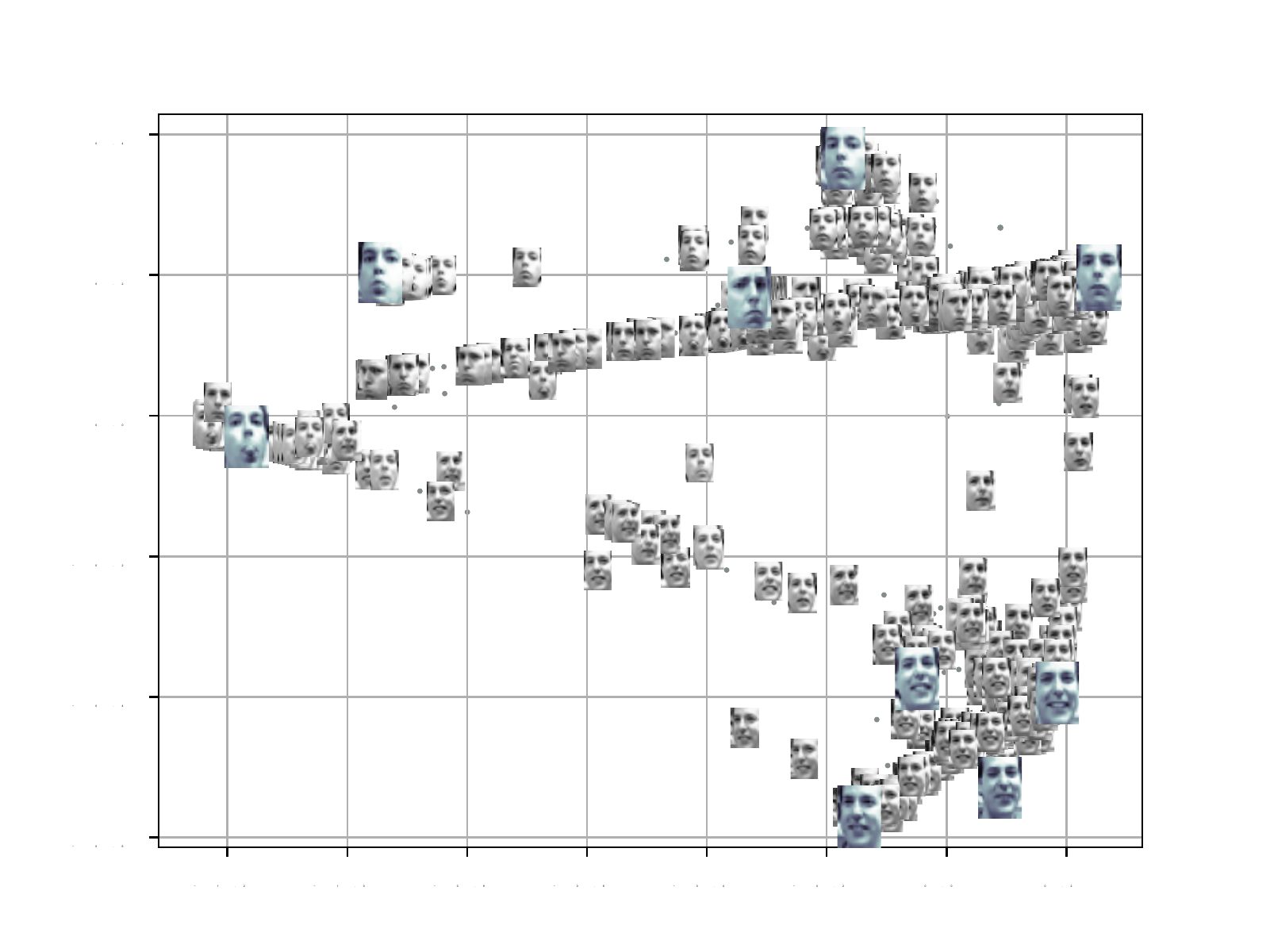} }}%
    \caption{Embedding of the Frey-face dataset. Larger images corresponds to landmarks.} 
    \label{fig:freyface_embedding}%
\end{figure}

\paragraph{MNIST experiment}
The MNIST data-set we use is taken from \cite[Accessed: 2022-09-30]{dataset_webpage} and consists of $60000$ $28\times 28$ images of handwritten digits. For our experiment, we extract $5000$ images from each of the digits $\{3,4\}$. From each digit, we also select at random $m=5$ landmarks. Fig. \ref{fig:mnist_embedding_nlm10} illustrates the embedding. In particular, Fig. \ref{subfig:mnist_embedding_10lm} shows how the embedding separates digit $3$ and $4$. Furthermore, we can see how the orientation/rotation of the digits determines their clustering, especially apparent for digit $4$. Selecting a subset of the landmarks, see Fig. \ref{subfig:mnist_local_vs_global_10lm}, we make a local embedding of their neighborhood as shown in Fig. \ref{subfig:mnist_local_10lm}.

\begin{figure}[htb!]
    \centering
    \subfloat[\label{subfig:mnist_embedding_10lm}\protect\centering MNIST embedding with 10 landmarks ]{{\includegraphics[width=0.32\textwidth]{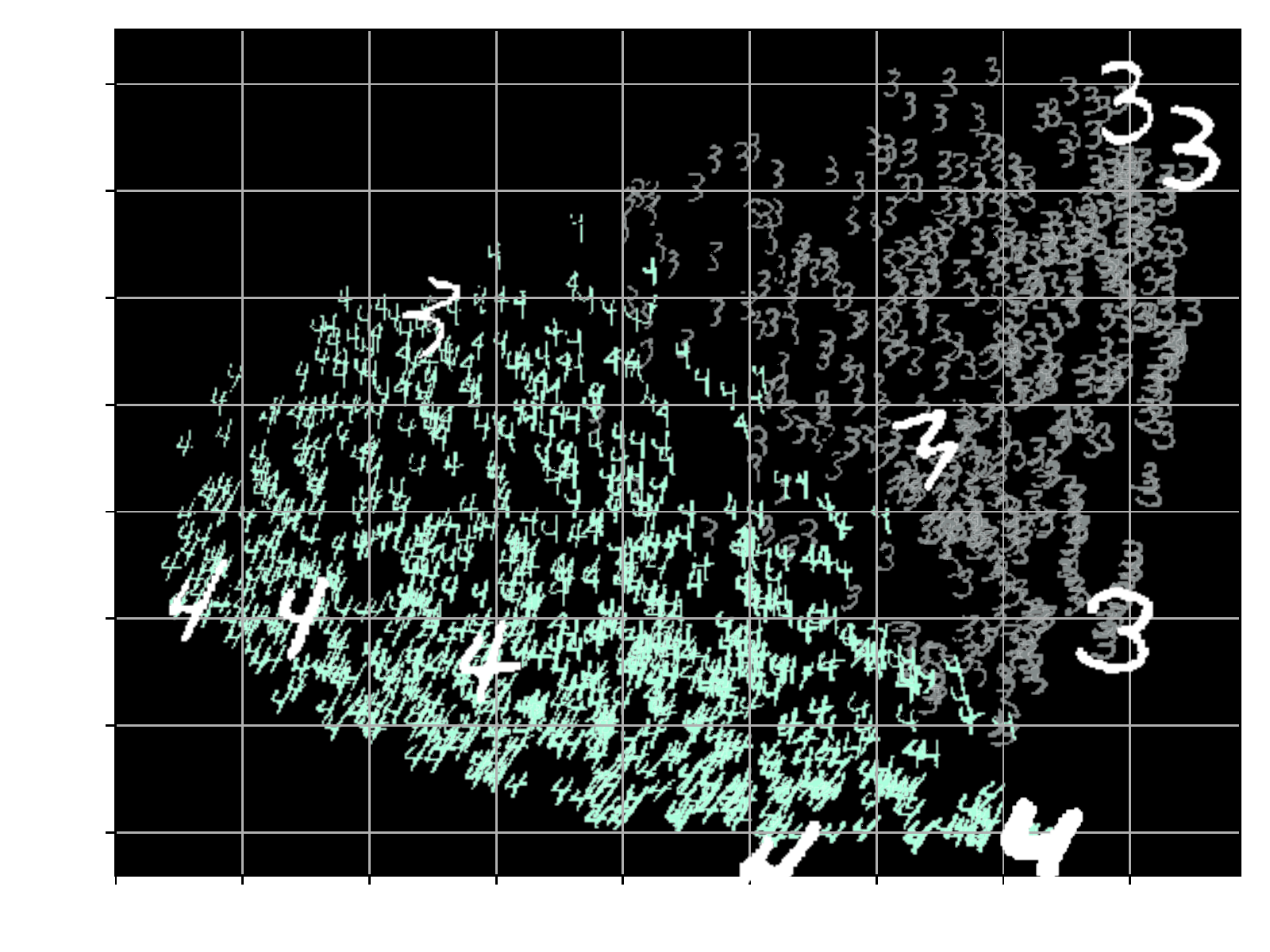} }}%
    \hfill
    \subfloat[\label{subfig:mnist_local_vs_global_10lm}\protect\centering Select a subset of landmarks and their neighborhood ]{{\includegraphics[width=0.32\textwidth]{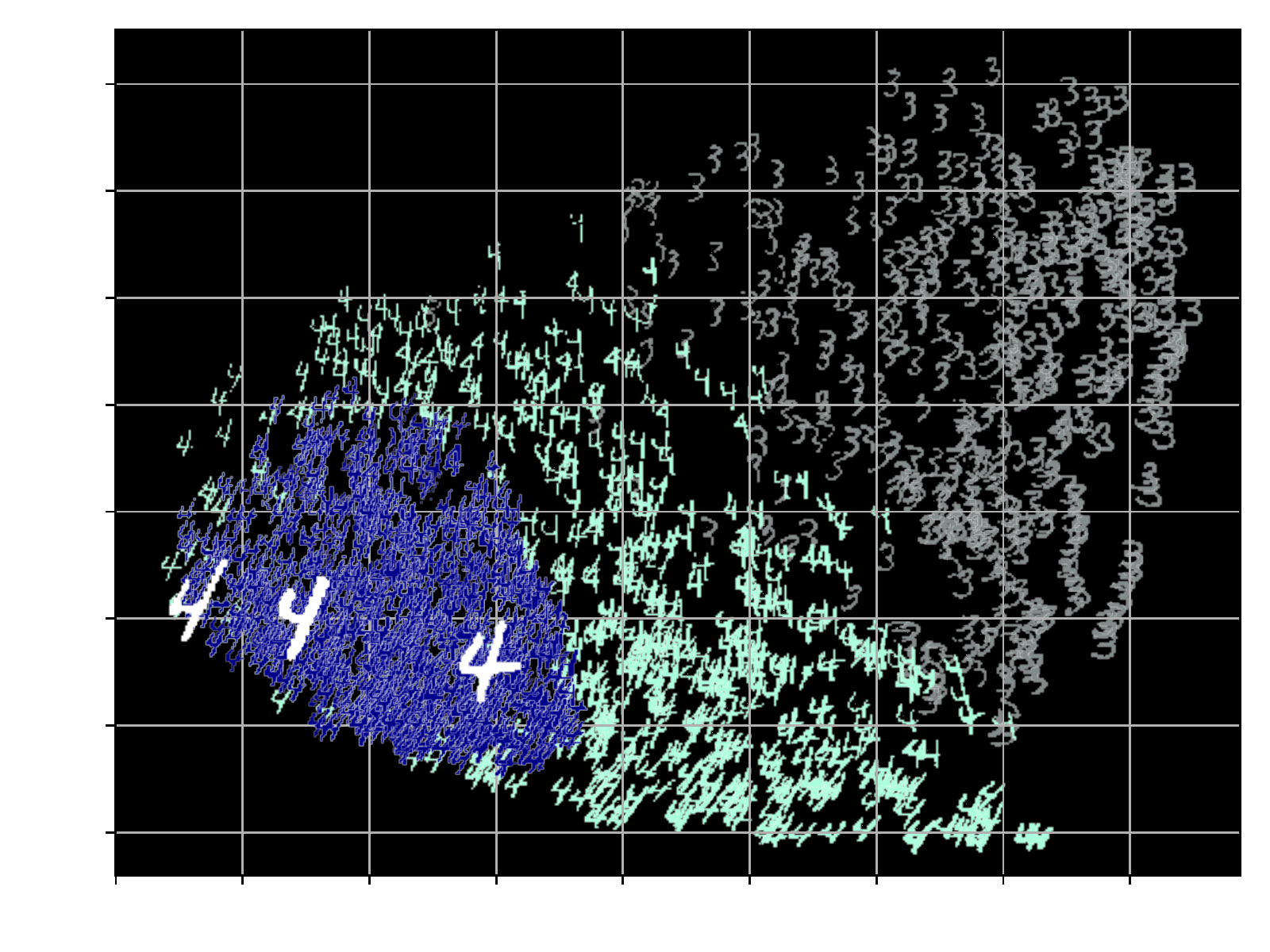} }}%
    \hfill
    \subfloat[\label{subfig:mnist_local_10lm}\protect\centering Local embedding of neighborhood of selected landmarks]{{\includegraphics[width=0.32\textwidth]{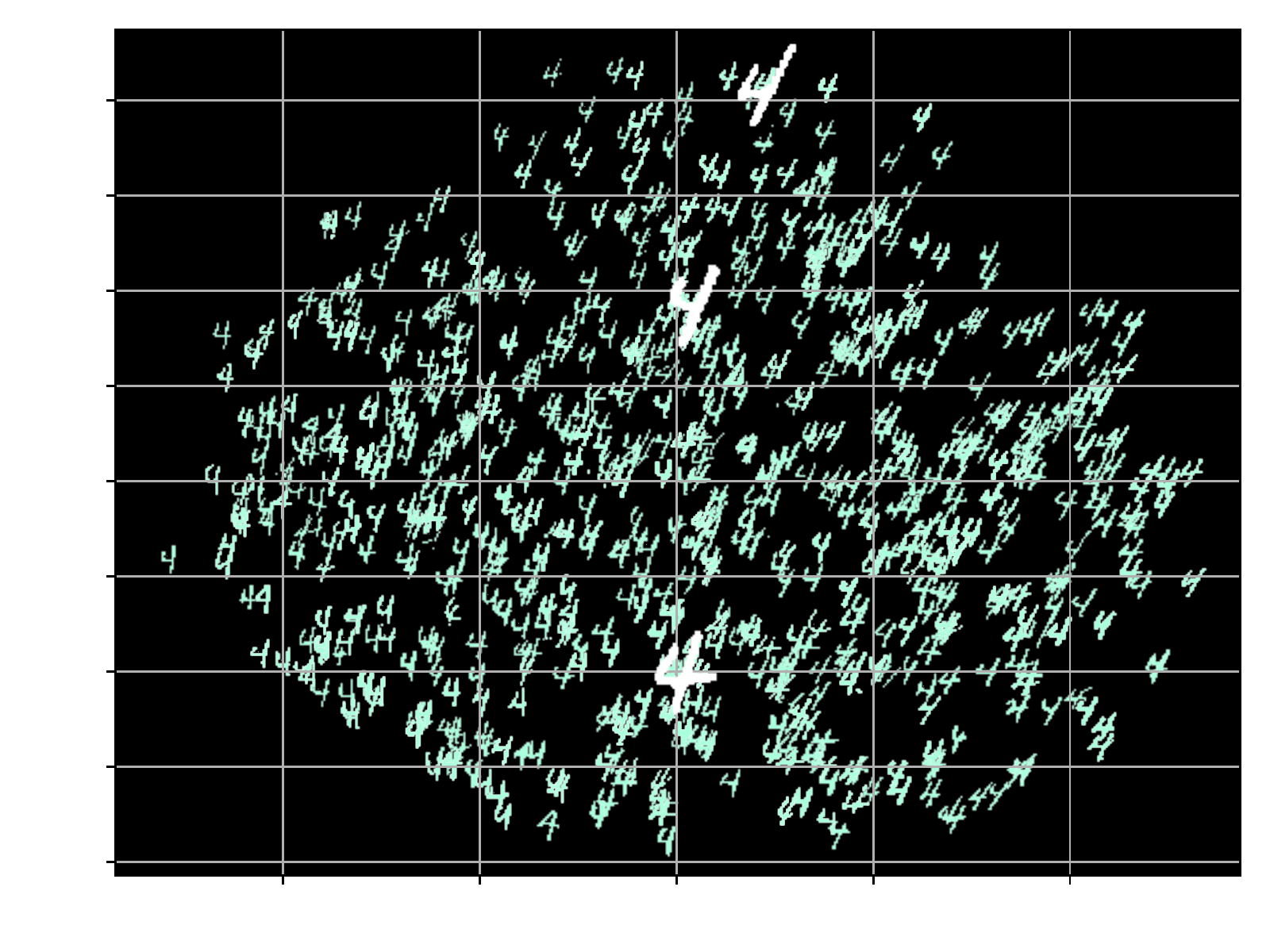} }}%
    \caption{Embedding of digit 3 and 4 from the MNIST dataset suing 10 landmarks and 5000 samples from each digit.} 
    \label{fig:mnist_embedding_nlm10}%
\end{figure}


\section{A note on computational efficiency} 
\label{section:A note on computational efficiency}
We have seen how the grounded metric voltage function decays exponentially with the distance from the source. This property can be thought of as a localization property, as the voltage will be negligible on most of the graph, except for finite support around the source.

The advantage of this localization is that when computing the voltage, only a subset of the graph needs to be used for any given source, reducing both memory and computational requirements. Importantly, the finite support of the voltage solution is a great alternative to traditional methods such as Laplacian eigenmaps, which rely on the computation of eigenfunctions that are typically global (meaning computations relies on the entire graph).

In the following, we characterize the effective support of the voltage solution and its dependency on the parameters of the grounded metric graph. Let $\M_1$ be a source region of radius $r$ as defined in Def. \ref{definition:grounded_metric_resistor_graph}. Furthermore, let $f : \M \rightarrow [0, 1]$ be a radially symmetric function over the metric space $(\M, d)$ such that $f(x) = h(d(x_1, x))$. For $\tau > 0$ and $x_1 \in \M_1$ we say that $f$ is localized around $x_1$ with support radius $r_{supp}$ if for $d(x_1, x) > r$,
\begin{equation*}
    h(x) \geq \tau \quad \text{for} \quad d(x_1, x) \leq r_{supp} \quad \text{and} \quad  h(x)< \tau \quad \text{for} \quad d(x_1, x) > r_{supp}.
\end{equation*}

\begin{corollary}(Locality of the voltage solution)
Consider a disk in $\bbR^d$. As a consequence of the exponential decay of the grounded metric voltage described in Theorem \ref{thm:single_landmark_bounds} and Corollary \ref{coro:voltage_shape_disk}, the voltage solution will be localized around the source region $\M_1$, with a support radius bounded by $r_{l} \leq r_{supp} \leq r_{u}$ where
\begin{equation*}
    r_{l} \coloneqq  \frac{\frac{r}{2}\log{ 1/\tau}}{\log{(Cr^d + \rho)/\Gamma}} \quad \text{and} \quad     r_{u} \coloneqq  \frac{r\log{ 1/\tau}}{\log{(1 + \rho/Cr^d)}}.
\end{equation*}
\label{corollary:span_of_support}
\end{corollary}

From Corollary \ref{corollary:span_of_support} it follows that the support of the voltage solution is restricted to a subset of the manifold, centered around the source. In particular, we see how the resistance to ground $1/\rho$ plays a crucial role in its relation to the effective support. Namely, as the resistance to ground decrease to zero, i.e. $\rho \rightarrow \infty$, then the effective support will also go to zero $r_{supp}\rightarrow 0$. The intuition here is that, with zero resistance to the ground, the ground drains all the current. Similarly, when the resistance to the ground goes to infinity, i.e. $\rho \rightarrow 0$ then the effective support goes to infinity $r_{supp}\rightarrow \infty$. This demonstrates the importance of the  grounding node in achieving localized effective support for the voltage solution.


\section{Conclusion and future work}
\label{section:Future work}
In this paper, we have demonstrated how the grounded metric voltage function holds promise as a tool for low-cost and distributed embedding of manifolds. The goal of this paper has been to establish a theoretical basis for further development of this embedding strategy. In particular, we have shown existence, convergence, and shape properties and demonstrated how the voltage captures local structure of the manifold.

There are three main directions for expanding on this work. First, we are interested in exploring the potential of constructing low-dimensional embeddings of data sets using grounded voltage functions. While we took the first steps in this direction in this paper (with theoretical results on the sphere and experiments over the sphere, MNIST, and Frey-faces), we are actively working towards more general results over arbitrary manifolds and more extensive real-world data sets. 

Second, we are interested in the computational aspect of this technique -- the local nature of the grounded voltage function suggests that distributed computation approaches are viable. Thus, one important problem is to develop algorithms that can provably and empirically leverage this. Finally, we are interested in describing the limiting operator of our approach.

\bibliographystyle{unsrt}
\bibliography{References}
\newpage
\appendix
\renewcommand{\thesection}{\Alph{section}}

\section{Proof of Lemma \ref{lemma:Solution_EMV_for_LVE}}

To prove Lemma \ref{lemma:Solution_EMV_for_LVE} we first include some notation. Let $s\in X^s$ be the source node, $g$ the ground node, and $e_i\in \bbR^{n+1}$ the indicator of node $i$.

\begin{proof}
\noindent The constraints $v(s) =1$ and $v(g) = 0$ can be written as $v^\top e_s=1$ and $v^\top e_g = 0$ respectively.  We apply these constraints using the Lagrangian multipliers $\lambda_s$ and $\lambda_g$, which gives $f(v) = v^\top L v - \lambda_s v^\top e_s - \lambda_g v^\top e_g$. When equating the derivative of $f$ to zero, and using $L v = (D -W)v = D(v-D^{-1}Wv)$ we can write this equation as 
\begin{equation}
    v = D^{-1}Wv + D^{-1} \lambda_s e_s - D^{-1}\lambda_g e_g.
    \label{eq:voltage_eq1}
\end{equation}
Since we enforce $v(s)=1$ on the source node, it follows that row $s$ in Eq. \eqref{eq:voltage_eq1} is $\lambda_s = d_{ss} - (Wv)_s$. Similarly, we have for row $g$ that $\lambda_g = (Wv)_g$. It follows, $v = \widetilde{D}^{-1}\widetilde{W}^{(s)}v$,
where
\begin{equation*}
    \widetilde{W}^{(s)}_{ij} = 
    \begin{cases}
     1, & \text{if} \quad i=j, x_i \in X^s \\
     0, & \text{if} \quad i\neq j , x_i \in X^s \\
     W_{ij}, & \text{otherwise},
    \end{cases}
\end{equation*}
and $\widetilde{D} \in \bbR^{n\times n}$ is a diagonal matrix with $D_{ii}=1$ for $x_i\in X^s$ and $\widetilde{D}_{ii} = \rho + \sum_{j = 1}^n w_{ij}$ otherwise. Note that the $n+1$'th row and column (row and column of the ground) are dropped, since the voltage on the ground is always zero anyway. If we have more source nodes in $X^s$, these can be incorporated similarly by applying Lagrangian multipliers $\lambda_i v_i e_i$ for each $i\in X^s$.
\end{proof}

\section{Proofs from Section \ref{sec:ground_analysis}}
\label{section:proofs_ground_analysis}

Throughout this section, we treat $M^s$ as fixed so as to avoid cumbersome notation addressing the source region. Furthermore, because $v_n^*$ and $v^*$ are defined to equal $1$ over $M^s$, we will reinterpret them as maps $M \setminus M^s \to [0, 1]$. Any such map $v$ can be immediately transformed into a map over all $M$ by including $v(x) = 1$ for all $x \in M^s$. 

\subsection{Proof of Theorem \ref{thm:existence_ground}}
We begin by expressing Proposition \ref{prop_ohms_law_result} with  an affine transformation.

\begin{definition}\label{defn:affine}
Let $\mu$ be a measure over $M$, and let $v: M \setminus M^s \to \reals$ be a measurable map. Then define $A_\mu v: M \setminus M^s \to \reals$ and $b_\mu: M \setminus M^s \to \reals$ as 
\begin{equation*}
    (A_\mu v)(x) = \frac{ \int_{M \setminus M^s} k(x,y)v(y) d\mu(y)}{\rho + \int_M k(x,y) d\mu(y)} \quad \text{and} \quad b_\mu(x) = \frac{\int_{M^s}k(x,y)d\mu(y) }{\rho + \int_M k(x,y) d\mu(y)}.
\end{equation*}
Together, we let $T_\mu v = A_\mu v + b_\mu.$
\end{definition}

Thus, to prove Theorem \ref{thm:existence_ground}, it suffices to show that for any choice of $\mu$, there exists a unique map $v^*: M \setminus M^s \to \R$ that satisfies $T_\mu v^* = A_\mu v^* + b_\mu$. To this end, the key idea will be to leverage that $A_\mu$ is a contraction.

\begin{definition}\label{defn:function_family}
Let $\mathcal{F}$ denote the space of measurable functions $f: M \to \reals$. Define $||f||_\infty$ as $$||f||_\infty = \sup_{x \in X} |f(x)|.$$ Furthermore, the $\ell_\infty$ distance between two functions $f, g$ is defined as $||f - g||_\infty.$
\end{definition}

It is well known that $\mathcal{F}$ is a closed metric space under the $\ell_\infty$ metric. We now show that $A_\mu$ is a contraction with respect to the $\ell_\infty$-metric. 

\begin{lemma}\label{lemma:contraction}
For any measurable function $f: M \to \reals$, $$||A_\mu f||_\infty \leq \frac{1}{1 + \rho} ||f||_\infty$$
\end{lemma}

\begin{proof}
This follows from algebraic manipulations. We have
\begin{equation*}
\begin{split}
\sup_{x \in M} |(A_\mu f)(x)| &= \sup_{x \in M \setminus M^s} \frac{ \int_{M \setminus M^s} k(x,y)f(y) d\mu(y)}{\rho + \int_M k(x,y) d\mu(y)} \leq \sup_{x \in M} \frac{ \int_{M \setminus M^s} k(x,y)||f||_\infty d\mu(y)}{\rho  + \int_M k(x,y) d\mu(y)} \\
&= ||f||_\infty \sup_{x \in M} \left( \frac{ \int_{M \setminus M^s} k(x,y) d\mu(y)}{\rho + \int_M k(x,y) d\mu(y)}\right)\leq ||f||_\infty \frac{1}{1 + \rho},
\end{split}
\end{equation*}
with the last inequality holding since $k$ has range $[0, 1]$ by assumption. 
\end{proof}

We now prove Theorem \ref{thm:existence_ground}.

\begin{proof}
Set $u_0: M \setminus M^s \to \reals$ as the $0$ function. That is $u_0(x) = 0$ for all $x$. For $i \geq 1$, define $u_i = A_\mu(u_{i-1}) + b_\mu$. We claim that this defines a Cauchy sequence over $\mathcal{F}$ (Definition \ref{defn:function_family}). To see this, observe that for any $i \geq 1$,
\begin{equation*}
\begin{split}
||u_{i+1} - u_i||_\infty &= ||(A_\mu u_i + b_\mu) - (A_\mu u_{i-1} + b_\mu)||_\infty = ||A_\mu(u_i - u_{i-1})||_\infty \leq C||u_i - u_{i+1}||_\infty
\end{split}
\end{equation*}
with the last inequality holding by Lemma \ref{lemma:contraction}. This implies that the distances between consecutive elements of our sequence decrease geometrically. Since $\mathcal{F}$ is closed with respect to $\ell_\infty$, it follows that this sequence converges to some function $v$. 

Next, we show $v$ satisfies $v=  A_\mu v + b_\mu$. Fix any $\epsilon > 0$. Then there exists $n$ such that $||v - u_n||_\infty, ||v - u_{n+1}||< \epsilon$. It follows that 
\begin{equation*}
\begin{split}
||v - (A_\mu v + b_\mu)||_\infty &\leq ||v - u_{n+1}||_\infty + ||(A_\mu v + b_\mu) - u_{n+1}||_\infty \\
&\leq \epsilon + ||(A_\mu v + b_\mu) - (A_\mu u_n + b_\mu)||_\infty = \epsilon + ||A_\mu(v - u_n)||_\infty \leq \epsilon + \frac{\epsilon}{1 + \rho}.
\end{split}
\end{equation*}
Since $\epsilon$ was arbitrary, it follows that $v = A_\mu v + b_\mu$. 

Next, to show $v$ has range $[0, 1]$, we simply observe that $u_n$ has range $[0, 1]$ for all $n$. This can be shown by induction on $n$. The base case clearly holds, and for the inductive step, observe that $(A_\mu v_n)(x), b_\mu(x) \geq 0$ $k$ is always non-negative which implies $v_{n+1}(x) = (A_\mu v_n)(x) + b_\mu(x) \geq 0$. To show that it is at most $1$, we have
\begin{equation*}
\begin{split}
v_{n+1}(x) &= (A_\mu v_n)(x) + b_\mu(x) = \frac{ \int_M k(x,y)v_n(y) d\mu(y)}{\rho + \int_M k(x,y) d\mu(y)} \leq  \frac{ \int_M k(x,y) d\mu(y)}{\rho  + \int_M k(x,y) d\mu(y)} \leq 1. 
\end{split}
\end{equation*}

Finally, to show uniqueness, suppose that $v'$ also satisfies $v' = A_\mu v' + b_\mu$. Then we have 
\begin{equation*}
\begin{split}
||v - v'||_\infty &= ||(A_\mu v + b_\mu) - (A_\mu v' + b_\mu)||_\infty = ||A_\mu(v- v')||_\infty \leq \frac{1}{1 + \rho}||v - v'||_\infty.
\end{split}
\end{equation*}
which implies $||v - v'||_\infty = 0$ as desired. 
\end{proof}

\subsection{Proof of Theorem \ref{thm:convergence_ground}}

We begin by observing that the extended voltage solution (Definition \ref{defn:extended_solution} for a finite sample $S$ exhibits similar behavior to the voltage solution over a measure $\mu$. The key idea is to define the measure $\mu_S$ over $M$ as the measure induced by the uniform distribution over $S$. 

\begin{definition}\label{defn:finite_measure}
Let $S = \{x_1, x_2, \dots, x_n\}$ be a finite set of $n$ points from $M$. Then $\mu_S$ is the measure on $M$ defined as $$\mu_S(P) = \frac{1}{n}\sum_{i=1}^n \ind(x_i \in P),$$ for all measurable sets $P$. 
\end{definition}

By substituting $\mu_S$ into Definitions \ref{defn:finite_measure} and \ref{defn:affine}, we have $$v_n^* = A_{\mu_S}v_n^* + b_{\mu_S}.$$ Here note that we are considering its restriction to $M - M^s$ as discussed in the beginning of this section. Furthermore, the existence and uniqueness of $v$ follow directly Theorem \ref{thm:existence_ground}. 

We now desire to show that as the sample size increases to infinity, $v_n^*$ pointwise converges towards $v^*$. To prove this, we begin by showing that for large values of $n$, $T_{\mu_S}$ serves as an approximation for $T_\mu$ when applied to $v^*$. 

\begin{lemma}\label{lem:converge_single_point}
Let $x \in M \setminus M^s$ be a point, and $\epsilon > 0$ be a real number. Then $$\Pr_{S \sim \mu^n}\left[\big |(T_{\mu}v^*)(x) - (T_{\mu_S}v^*)(x)\big | > \epsilon \right] < 4 \exp(-\frac{n\rho^2\epsilon^2}{9}).$$ 
\end{lemma}

\begin{proof}
Let $y \sim \mu$, and $Y = k(x,y)v^*(y)\ind(y \notin M^s) + k(x,y)\ind(y \in M^s).$ Let $Y_1, Y_2, \dots, Y_n$ are i.i.d copies of $Y$.  Since $k, v^*$ both have ranges in $[0, 1]$, it follows that $Y$ has range $[0, 1]$ as well. It follows by Hoeffding's inequality that 

\begin{equation*}
\begin{split}
\Pr \left[\left| \mathbb{E}[Y] - \frac{1}{n}\sum_{i = 1}^n Y_i \right| > \rho \epsilon\right] < 2 \exp \left( -n\rho^2 \epsilon^2\right).
\end{split}
\end{equation*}

Similarly, we let $Z$ denote the random variable $k(x, y)$, $y \sim \mu$ and $Z_1, \dots, Z_n$ be i.i.d copies of $Z$. We also have 
\begin{equation*}
\begin{split}
\Pr \left[\left| \mathbb{E}[Z] - \frac{1}{n}\sum_{i = 1}^n Z_i \right| > \rho \epsilon\right] < 2 \exp \left( -n\rho^2 \epsilon^2\right).
\end{split}
\end{equation*}
Next, we express $(T_\mu v^*)(x)$ and $(T_{\mu_S} v^*)(x)$ in terms of these variables. We have,
\begin{equation*}
\begin{split}
(T_\mu v^*)(x) &= \frac{\int_{M^s} k(x,y)d\mu(y) + \int_{M \setminus M^s} k(x,y)v^*(y) d\mu(y)}{\rho  + \int_M k(x,y) d\mu(y)} = \frac{\Ev[Y]}{\rho + \Ev[Z]}.
\end{split}
\end{equation*}
Similarly, $$(T_{\mu_S}v^*)(x) = \frac{\frac{1}{n}\sum Y_i}{\rho + \frac{1}{n} \sum Z_i}.$$

We will now use these expressions to bound the difference between $(A_\mu u + b_\mu)(x)$ and $(A_{\mu_S}(u) + b_{\mu_S})(x)$ in terms of $Z, Y, Z_i, Y_i$. For convenience, let $$\Delta = \max\left(|\frac{1}{n}\sum Y_i - \Ev[Y]|, |\frac{1}{n}\sum Z_i - \Ev[Z]|\right).$$ Then we have 
\begin{equation*}
\begin{split}
|(A_\mu u)(x) - (A_{\mu_S}u)(x)| &= \left|\frac{\Ev[Y]}{\rho + \Ev[Z]} - \frac{\frac{1}{n}\sum Y_i}{\rho + \frac{1}{n} \sum Z_i} \right|. 
\end{split}
\end{equation*}
To bound this, we split the difference into two parts. We have
\begin{equation*}
\begin{split}
\left|\frac{\Ev[Y]}{\rho + \Ev[Z]} - \frac{\frac{1}{n}\sum Y_i}{\rho + \Ev[Z]}\right| &\leq \frac{\Delta}{\rho},
\end{split}
\end{equation*}

\begin{equation*}
\begin{split}
\left|\frac{\frac{1}{n} \sum Y_i}{\rho + \Ev[Z]} - \frac{\frac{1}{n}\sum Y_i}{\rho + \frac{1}{n} \sum Z_i}\right| &\leq \frac{(\frac{1}{n}\sum Y_i)\Delta}{(\rho + \Ev[Z])(\rho + \frac{1}{n}\sum Z_i)} \leq \frac{\Delta}{\rho}.
\end{split}
\end{equation*}

Similarly, we can also show that $|(b_\mu - b_{\mu_S})(x)| \leq \frac{\Delta}{\rho}.$ Applying a union bound for the deviations of $Y$ and $Z$, we see that $\Delta > \frac{\rho}{3}$ with probability at most $4 \exp(-\frac{n\rho^2\epsilon^2}{9})$.
\end{proof}

Next, we show how to adapt Lemma \ref{lem:converge_single_point} to hold uniformly over the entire sample $S$. To do so, we first define a type of metric over the space of functions on $M$.

\begin{definition}\label{defn:s_norm}
Let $S = \{x_1, \dots, x_n\} \subset M$ be a set of points, and let $u: M \setminus M^s \to \R$ be a function. Then $||u ||_S = \max_{x_i \in M \setminus M^s} |u(x_i)|$ is the largest absolute value of $u$ over $S$. 
\end{definition}

\begin{lemma}\label{lem:converge_sample}
Let $S = \{x_1, \dots, x_n\} \sim \mu^n$, and $\epsilon > 0$ be a real number. Then for $n > \frac{6}{\epsilon \rho}$, $$\Pr_{S \sim \mu^n}\left[||v^* - T_{\mu_S}v^*||_S > \epsilon\right] < 4n \exp\left(-\frac{(n-1)\rho^2\epsilon^2}{36}\right).$$
\end{lemma}

\begin{proof}
Fix $x_i \in M \setminus M^s$. It suffices to show that $|(T_\mu v^*)(x_i) - (T_{\mu_S}v^*)(x_i)| > \epsilon$ with probability at most $4\exp \left(-\frac{(n-1)\rho^2\epsilon^2}{36}\right)$. To do so, we essentially apply Lemma \ref{lem:converge_single_point}. The only difficulty is that $S$ and $x_i$ are no longer independent as $x_i \in S$. To resolve this, we observe that $x_i$ is independent from $S \setminus x_i$, and use this to bound the difference. Applying this, we have, 
\begin{equation*}
\begin{split}
|(T_\mu v^*)(x_i) - (T_{\mu_{S \setminus x_i}}v^*)(x_i)| &\leq \left|(T_\mu v^*)(x_i) - (T_{\mu_{S \setminus x_i}}v^*)(x_i)\right| + \left|(T_{\mu_{S \setminus x_i}}v^*)(x_i) - (T_{\mu_S}v^*)(x_i) \right|.
\end{split}
\end{equation*}
The former term can be directly bounded using Lemma \ref{lem:converge_single_point}. Because $x_i$ and $S \setminus x_i$ are independent, we have that with probability at most $4\exp \left(-\frac{(n-1)\rho^2\epsilon^2}{36}\right)$, 
\begin{equation*}
\left|(T_\mu v^*)(x_i) - (T_{\mu_{S \setminus x_i}}v^*)(x_i)\right| > \frac{\epsilon}{2}.
\end{equation*} It thus suffices to show that the latter term is at most $\frac{\epsilon}{2}$. To do so, we split $Tv^*$ into $Au + b$ for both $\mu_S$ and $\mu_{S \setminus x_i}.$ Using the same variables, $Y, Y_j, Z, Z_j$ as in the proof of Lemma \ref{lem:converge_single_point} (and substituting $x = x_i$), we have 
\begin{equation*}
\begin{split}
\left|A_{\mu_{S \setminus x_i}}v^*(x_i) - A_{\mu_S}v^*(x_i) \right| &= \left|\frac{\frac{1}{n-1}\sum_{j \neq i} Y_j}{\rho + \frac{1}{n-1} \sum_{j \neq i} Z_j} - \frac{\frac{1}{n}\sum Y_j}{\rho  + \frac{1}{n} \sum Z_j} \right| \\
&\leq \left|\frac{\frac{1}{n} + \frac{1}{n}\sum Y_j}{\rho   -\frac{1}{n} + \frac{1}{n} \sum Z_j} - \frac{\frac{1}{n}\sum Y_j}{\rho  + \frac{1}{n} \sum Z_j} \right| \leq \frac{2}{n \rho},
\end{split}
\end{equation*}
with the last inequality coming from the same manipulations applied in the proof of Lemma \ref{lem:converge_single_point}. We can similarly show that $$\left|b_{\mu_{S \setminus x_i}}(x_i) - b_{\mu_S}(x_i) \right| \leq \frac{1}{n\rho}.$$ However, since $n > \frac{6}{\epsilon\rho}$, it follows that $|T_{\mu_{S \setminus x_i}}(u)(x_i) - T_{\mu_S}(u)(x_i)| \leq \frac{\epsilon}{2},$ as desired. 
\end{proof}

We are now prepared to prove Theorem \ref{thm:convergence_ground}.

\begin{proof}
Let $S \sim \mu^n$ be a sample of $n$ i.i.d points, and let $v_n^*$ be its corresponding voltage solution. By using natural analogs to Theorem \ref{thm:existence_ground} and Lemma \ref{lemma:contraction}, we immediately have that for any map $u: M \to [0, 1]$, $||A_{\mu_S}u||_S \leq \frac{1}{1+\rho}||u||_S$, and the sequence $\{T_{\mu_S}^mu\}_{m \geq 1}$ converges pointwise to $v_n^*$. We also have that the norm induced by Definition \ref{defn:s_norm} induces a metric over the space of maps $M \setminus M^s \to \R$.

Next, by Lemmas \ref{lem:converge_single_point} and \ref{lem:converge_sample}, with probability $1 - 2\exp(O(-n))$ over $S$, both of their desired bounds hold for some given point $x \in M$. Using this, along with the fact that $v^*, v_n^*$ are the fixed points of $T_\mu$, $T_{\mu_S}$, we have,
\begin{equation*}
\begin{split}
|v^*(x) - v_n^*(x)| &= |T_\mu v^*(x) - T_{\mu_S}v_n^*(x)| \leq |T_\mu v^*(x) - T_{\mu_S} v^*(x)| + |T_{\mu_S} v^*(x) - T_{\mu_S}v_n^*(x)| \\
&\leq \epsilon + |A_{\mu_S}v^*(x) - A_{\mu_S}v_n^*(x)| \leq \epsilon + \frac{\sum_{x_i \in M \setminus M^s} k(x, x_i)|v^*(x_i) - v_n^*(x_i)|}{\rho + \sum_{x_i} k(x, x_i)} \\
&\leq \epsilon + |v^* - v_n^*|_S \leq \epsilon + |v^* - T_{\mu_S}v^*|_S + \sum_{i= 1}^\infty |T_{\mu_S}^iv^* - T_{\mu_S}^{i+1}v^*|_S \leq \epsilon + \frac{\epsilon}{\rho},
\end{split}
\end{equation*}
Since $\epsilon$ is arbitrary, the claim follows as $\rho$ is a fixed constant. 
\end{proof}

\section{Proofs on the shape and support of the GMV}

\subsection{Proof of Theorem \ref{thm:single_landmark_bounds}}
In this section, we prove Theorem \ref{thm:single_landmark_bounds}. First, we establish existence, radial symmetry, and monotonicity regarding $\lambda$. Let $A_\nu$ and $b_\nu$ be as defined in Definition \ref{defn:affine}, with $M = S^{d-1}$ and $M^s = B(0, r_s)$. The existence of $\lambda(x) =v^*(x)$, where  $v^*(x) = (A_\mu v^*)(x) + b_\mu$, then follows straightforwardly from Theorem \ref{thm:existence_ground}. 
Meanwhile, from Section \ref{section:proofs_ground_analysis}, we have already established that $v^*(x)$ corresponds to a local average of its neighbors, which implies that $\lambda$ must be strictly non-increasing away from the source. 
Furthermore, due to the radial symmetry of the sphere and the kernel $k$, it follows that also $A_\mu$ and  $b_\mu(x)$ are radially symmetric. Moreover, in Section \ref{section:proofs_ground_analysis} we establish that $v_{n+1}(x) = (A_\mu v_n)(x) + b_\mu$ converges to $v^*$. With $v_0$ radially symmetric, it follows that also $\lambda$ must be radially symmetric. Finally, we examine the upper and lower bounds. 
\begin{proof}
(Theorem \ref{thm:single_landmark_bounds}) We have already shown that $\lambda$ exists and is radially symmetric (meaning $h$ exists) and that $h$ is strictly non-increasing. We now prove the bounds on $h$.

\paragraph{Upper bound:}
Let $x$ be a point on the sphere, and $z=d_m(x_s, x)$  be the geodesic distance from the source landmark $x_s$. Furthermore, since we consider the unit sphere, we have for two points on the sphere ,whose euclidean distance is $r$, that $d_M(x_i, x_j) = \arccos(\langle x_i, x_j\rangle) = \phi(r)$.

The key idea is now to bound the integral, $\int_M k(x, y)v(y)d\mu(y)$. To do so, observe that by the definition of $k$, this integral is only non-zero over the ball $B(x, r)$. Furthermore, at most half of the probability mass of this ball satisfies $v(y) \geq v(x)$, and all points inside this ball satisfy $v(y) \leq h(z - \phi(r))$ as $z-\phi(r)$ is the closest that any point in this ball gets to the origin. Thus, bounding the expectation, we have 
$$\int_M k(x, y)v(y)d\mu(y) \leq \frac{a}{2}h(z-r) + \frac{a}{2}h(z).$$ Substituting this, we have that 
\begin{equation*}
\begin{split}
h(z) &= v(x) = T_\mu v (x) = \frac{\int_M k(x, y)v(y)d\mu(y)}{\rho + a} \leq \frac{\frac{a}{2}h(z-\phi(r)) + \frac{a}{2}h(z)}{\rho + a}.
\end{split}
\end{equation*}

With $z = z^\prime + \phi(r)$ the upper bound can be written as $h(z+\phi(r)) \leq C_u h(z)$ where $C_u = 1/(1+2\rho/a)$. Recursion of this expression gives $h(z+t \phi(r)) \leq C_u^t h(z)$. We now let $z = z_1$ where $z_1$ defines the center of the source, since $h(z_1) = 1$ this gives $  h(z + t \phi(r)) \leq \exp{(-t\ln{(1+2\rho/a)})}$.

\paragraph{Lower Bound:}
We use a similar strategy as we did with the upper bound. This time, we let $\Gamma$ the probability mass of the ball $B(x, r)$ that consists of points $y$ for which $||y|| \leq z - \phi(r/2)$. While this value depends on $z$, it can be lower bounded by the case in which $||z|| = \phi(r)$. This constant thus equals the intersection volume between 2 $d$-dimensional spheres. Using this constant, we see that $$\int_M k(x, y)v(y)d\mu(y) \geq \Gamma h(z-\phi(r/2)).$$ Substituting this, we have 
\begin{equation*}
\begin{split}
h(z) &= v(x) = T_\mu v (x) = \frac{\int_M k(x, y)v(y)d\mu(y)}{\rho + a} \geq \frac{\Gamma h(z - \phi(r/2))}{\rho + a},
\end{split}
\end{equation*}
We now treat the lower bound similarly to the upper bound. Here $h(z + t \phi(r/2)) \geq C_l^t$ where $C_l = \Gamma/(a+\rho)$, which gives $h(z + t \phi(r/2)) \geq \exp{(-t\ln{((a+\rho)/\Gamma)}}$.
\end{proof}

\subsection{Proof of Corollary \ref{corollary:span_of_support}} Consider a disk in $\bbR^d$. Let $z = z_1$ where $z_1$ defines the center of the source, $h(z_1) = 1$.
We are interested in the radial distance $r_{supp}$ from $z_1$ such that $h(z_1 + r_{supp}) \geq \tau$. Let $r_{supp} = tr$. From Corollary \ref{coro:voltage_shape_disk} we then have that
\begin{equation*}
    h_u(r_{supp}) \coloneqq \exp{\bracket{-\frac{r_{supp}}{r}\ln{(1+2\rho/a)}}} \quad \text{and} \quad h_l(r_{supp}) \coloneqq \exp{\bracket{-2\frac{r_{supp}}{r}\ln{((a+\rho)/\Gamma}}}
\end{equation*}

The upper bound for $r_{supp}$ should therefore be $r_{s,u} = \max\{r_{supp} : h_u(z_1 + r_{supp}) \geq \tau\}$. Similarly the lower bound should be $r_{s,l} = \max\{r_{supp} : h_l(z_1 + r_{supp}) \geq \tau\}$. Solving for $r_{supp}$ we find that $h_l(r_{supp}) \geq \tau $ and $h_u(r_{supp}) \geq \tau$ implies
\begin{align*}
    \begin{split}
         r_{supp} \leq \frac{\frac{r}{2}\log{ 1/\tau}}{\log{(Cr^d + \rho)/\Gamma}} \quad \text{and} \quad r_{supp} \leq \frac{r\log{ 1/\tau}}{\log{(1 + \rho/Cr^d)}}.
    \end{split}
\end{align*}
respectively. Here we used that $a = Cr^d$ as it is the volume of a $d$ dimensional sphere. From the definition of $r_{s,l}$ and $r_{s,u}$ the result follows.

\end{document}